%% file: Preprint_arxiv.tex
\documentclass[11pt]{article}
\usepackage[margin=1in]{geometry} 
\usepackage{graphicx} 
\usepackage{authblk}

\usepackage{color}
\usepackage{bbm}
\usepackage{hyperref}

\usepackage{amsmath}
\usepackage{amssymb}
\usepackage{mathtools}
\usepackage{amsthm}
\usepackage{stmaryrd}
\usepackage{tabularx}

\input{packages.tex}
\input{commandes_maths}

\title{The Strong, Weak and Benign Goodhart's law. \\ An  independence-free and paradigm-agnostic formalisation}

\author[1]{Adrien Majka}
\author[1]{El-Mahdi El-Mhamdi}

\affil[1]{École Polytechnique, France}

\begin{document}

\maketitle

\input{abstract}

\input{Intro}

\input{related-work}

\input{model}

\input{results}

\input{discussion}

\input{ackowledgment}

\bibliographystyle{alpha}
\bibliography{biblio.bib}

\newpage
\onecolumn
\input{appendix}

\end{document}

%% file: packages.tex
\usepackage[square,numbers]{natbib}
\usepackage{url}            
\usepackage{hyperref}       
\usepackage{cleveref}
\usepackage[T1]{fontenc}
\usepackage[utf8]{inputenc}
\usepackage[english]{babel}
\usepackage{numprint}
\usepackage{eurosym}
\usepackage{appendix}
\usepackage{xspace}
\usepackage{xcolor}
\usepackage{verbatim}
\usepackage{authblk}

\usepackage{diagbox} 
\usepackage{multirow}
\usepackage{makecell}

\usepackage{fancyhdr}

\usepackage{etoolbox}

\usepackage{appendix}

\usepackage{graphicx}
\usepackage{graphics}   
\usepackage{wrapfig}
\usepackage[justification=centering]{caption}
\usepackage{subcaption} 
\usepackage{tikz}     
\usetikzlibrary{shapes.geometric, arrows}
\usepackage[framemethod=TikZ]{mdframed}

\usepackage{xifthen}
\usepackage{xargs}

\usepackage[maxlevel=3]{csquotes}
\usepackage{setspace}
\usepackage{siunitx} 
\usepackage{amsmath}        
\usepackage{mathrsfs}       
\usepackage{amssymb}        
\usepackage{amsfonts}       
\usepackage{mathtools}      
\usepackage{dsfont}         
\usepackage{stmaryrd}       
\usepackage{esvect}         
\usepackage{systeme}        
\usepackage{amsthm}

\usepackage{tcolorbox}
\usepackage{enumitem}
\usepackage{witharrows}

%% file: commandes_maths.tex
\makeatletter
\newcommand{\newreptheorem}[2]{\newtheorem*{rep@#1}{\rep@title}\newenvironment{rep#1}[1]{\def\rep@title{#2 \ref*{##1}}\begin{rep@#1}}{\end{rep@#1}}} 
\makeatother

\newtheorem{theorem}{Theorem}[section]
\newreptheorem{theorem}{Theorem}

\newtheorem{lemma}{Lemma}[section]
\newreptheorem{lemma}{Lemma}

\crefname{lemma}{Lemma}{Lemmas}

\let\ensembleNombre\mathbb
\newcommand*\N{\ensuremath{\ensembleNombre{N}}}       
\newcommand*\R{\ensuremath{\ensembleNombre{R}}}       

\newcommand{\bracks}[1]{\left\lbrack #1 \right\rbrack} 
\newcommand{\pars}[1]{\left( #1 \right)}               
\newcommand{\abs}[1]{\left\vert #1\right\vert}         

\newcommand*\intervalle[4]{\left#1 #2 \, ; #3 \right#4}                       
\newcommand*\intervalleEntier[2]{\intervalle{\llbracket}{#1}{#2}{\rrbracket}} 

\newcommandx{\proba}[3][1, 3=0]{
\ifthenelse{\isempty{#2}}{\mathbb{P}_{#1}}{%
\ifthenelse{\equal{#3}{0}}{\mathbb{P}_{#1}\pars{#2}}{\mathbb{P}_{#1}\pars{#2 \middle| #3}}}%
}
\newcommandx{\esp}[3][1, 3=0]{
\ifthenelse{\isempty{#2}}{\mathbb{E}_{#1}}{%
\ifthenelse{\equal{#3}{0}}{\mathbb{E}_{#1}\bracks{#2}}{\mathbb{E}_{#1}\bracks{#2 \middle| #3}}}%
}
\newcommandx{\var}[3][1, 3=0]{
\ifthenelse{\isempty{#2}}{\operatorname{Var}_{#1}}{%
\ifthenelse{\equal{#3}{0}}{\operatorname{Var}_{#1}\pars{#2}}{\operatorname{Var}_{#1}\pars{#2 \middle| #3}}}%
}
\newcommandx{\cov}[4][1, 4=0]{
\ifthenelse{\isempty{#2} \AND \isempty{#3}}{\operatorname{Cov}_{#1}}{%
\ifthenelse{\equal{#4}{0}}{\operatorname{Cov}_{#1}\pars{#2, #3}}{\operatorname{Cov}_{#1}\pars{#2, #3 \middle| #4}}}%
}
\newcommandx{\corr}[3][1]{\ifthenelse{\isempty{#2} \AND \isempty{#3}}{\operatorname{Corr}_{#1}}{\operatorname{Corr}_{#1}\pars{#2, #3}}}

\newcommandx{\Norm}[3][1]{\mathcal{N}_{#1}\pars{#2, #3}} 

\newcommandx{\egalloi}{\overset{\mathrm{\mathcal{L}oi}}{=}} 
\newcommandx{\egalprob}{\overset{\mathbb{P}}{=}}            
\newcommandx{\egaltxt}[1]{\overset{\mathrm{#1}}{=}}         
\newcommandx{\simiid}{\overset{\mathrm{iid}}{\sim}}         
\newcommand{\indep}{\perp \!\!\! \perp}
       
\newcommand{\1}[1][]{\mathds{1}\{#1\}}

\newcommandx{\maxx}[1]{\underset{#1}{\max}}
\newcommandx{\minn}[1]{\underset{#1}{\min}}
\newcommandx{\argmax}[1][1]{\underset{#1}{\arg\max}}
\newcommandx{\argmin}[1][1]{\underset{#1}{\arg\min}}

\newcommandx{\integ}[4][2]{\int_{#1}^{#2} #3 \, \mathrm{d} #4} 

\newcommandx{\surf}[2][2=m]{\numprint{#1}\,#2\textsuperscript{2}}
\newcommandx{\vol}[2][2=m]{\numprint{#1}\,#2\textsuperscript{3}}

\newcommand*\fonction[5]{
#1 \colon \left\{\begin{alignedat}{2}  &#2 &\: &\to      #3\\
                                &#4 &   &\mapsto  #5
\end{alignedat} \right. \kern-\nulldelimiterspace} 

\makeatletter
\newcommand{\leqnomode}{\tagsleft@true\let\veqno\@@leqno}
\newcommand{\reqnomode}{\tagsleft@false\let\veqno\@@eqno}
\makeatother

%% file: abstract.tex
\begin{abstract}
Goodhart's law is a famous adage in policy-making that states that ``When a measure becomes a target, it ceases to be a good measure''. 
As machine learning models and the optimisation capacity to train them grow, growing empirical evidence reinforced the belief in the validity of this law without however being formalised.
Recently, a few attempts were made to formalise Goodhart's law, either by categorising variants of it, or by looking at how optimising a \emph{proxy metric} affects the optimisation of an \emph{intended goal}.
In this work, we alleviate the simplifying independence assumption, made in previous works, and the assumption on the learning paradigm made in most of them, to study the effect of the coupling between the proxy metric and the intended goal on Goodhart's law. Our results show that in the case of light tailed goal and light tailed discrepancy, dependence does not change the nature of Goodhart's effect. However, in the light tailed goal and heavy tailed discrepancy case, we exhibit an example where over-optimisation occurs at a rate inversely proportional to the heavy tailedness of the discrepancy between the goal and the metric. 

\end{abstract}

%% file: Intro.tex
\section{Introduction}


From Charles Goodhart's remark in the context of monetary economics~\cite{goodhartMonetaryRelationshipsView1975} to its reformulation by Keith Hoskin \cite{hoskinAwfulIdeaAccountability1996} and its popularisation by Marylin Strathern\cite{strathernImprovingRatingsAudit1997}, Goodhart's law remained unformalised. However with the growing place of metrics, a growing body of work expresses the need for a better understanding of how Goodhart's law affects the reliance on metrics.
To cite a few example : the AI governance literature\cite{dotanEvolvingAIRisk2024, birkstedtAIGovernanceThemes2023,thomasRelianceMetricsFundamental2022}, fairness literature \cite{weertsAreThereExceptions2022}, alignment research \cite{hsiaGoodhartsLawApplies2023,thomasRelianceMetricsFundamental2022}.
As a response, several attempts has been made to formalise Goodhart's law to devise loss that can be provably resilient to it. Our work build on the work of \cite{el-mhamdiGoodhartsLawApplication2024} by taking the same formalisation but alleviating the independence hypothesis of their results.  

Our key contributions are the following : 
\begin{itemize}
    \item We alleviate the independence hypothesis of previous work.
    \item We derive a general result independent from the dependence structure with a general tail structure.
    \item We conduct a detailed analysis in two cases of dependence which highlight the importance of coupling.
    \item We propose a formal characterisation of different Goodhart's law effect that captures the different strenght of Goodhart's law. 
    \item For reproducibility purposes,  we verified the computations in our proofs with the \emph{symbolic computation} language Sympy~\cite{meurerSymPySymbolicComputing2017} and provide the code to do so in the supplementary materials~\footnote{LLMs are stricly not involved in any part of this paper}.
\end{itemize}

\paragraph{Paper structure.} The rest of this paper is organised as follows. In Section~\ref{sec:related_work}, we discuss some of the other work that formalised Goodhart's law. Section~\ref{sec:model} show our formal setup, in Section~\ref{sec:results} we first provide a comprehensive overview of our results, followed by their formal statements and sketches of proofs (detailed proofs are available in the Appendix), finally, Section~\ref{sec:discussion} discuss our results and proposes a comprehensive research agenda for future work.

%% file: related-work.tex
\section{Related Work}\label{sec:related_work}

\paragraph{Formalisations on Reinforcement Learning.}

RL was naturally a favourite setting to formalise Goodhart-Strathern's law as the most notable cases of reward-hacking in AI appeared in reinforcement learning (RL) settings~\cite{clarkFaultyRewardFunctions2016, amodeiConcreteProblemsAI2016, gaoScalingLawsReward2022}. The first part of \cite{kwaCatastrophicGoodhartRegularizing2024} shows the inefficiency of KL divergence to prevent reward hacking by showing that with heavy-tailed policy reward, a policy with arbitrarily high proxy reward but low penalty and low true reward always exists. \cite{karwowskiGoodhartsLawReinforcement2023} provide geometric explanation of Goodhart's law in reinforcement learning while devising provably Goodhart-resilient early stopping criterion. \cite{skalseDefiningCharacterizingReward2022} introduces a formalisation of reward hacking in the context of policy learning in RL. It proves several results on the general existence of reward hacking policy with respect to different reward function on different sets of policies.

\paragraph{Formalisations on Supervised Learning.}

\cite{hennessyGoodhartsLawMachine2020} intends to give a microfoundation to ML model to make them robust to the seminal critique of classic Keynesian models \cite{lucasEconometricPolicyEvaluation1976}. A regulator tries to make a prediction in a setting where, at test time, covariate can be manipulated by an agent to induce a more favorable decision from the regulator.

\paragraph{Paradigm-Agnostic Formalisations.}

Two precedent works by D.Manheim and S.Garrabrant \cite{manheimBuildingLessflawedMetrics2023,manheimCategorizingVariantsGoodharts2018} provide interesting insight on general metrics design, by providing a towering view on metric potential flaw and set of case separation on Goodhart's law respectively, while not fully formalising the problem. \cite{NEURIPS2020_b607ba54} devises a general framework for AI overoptimisation and draw results in the case of constrained ressources and partially specified goal. Their setup is inspired by incomplete contracting, but its reliance on state-space descriptions makes it more suitable to reinforcement learning. Previous works \cite{el-mhamdiGoodhartsLawApplication2024, kwaCatastrophicGoodhartRegularizing2024} provide a general formalisation of Goodhart's law, where no assumption is made on the learning paradigm and show nuanced cases where Goodhart's law holds or not depending on the relative thickness of the tail of the goal and the discrepancy. This was however done while assuming the discrepancy is independent from the goal. In our work, we follow the same paradigm-agnostic approach, but we alleviate the independence assumption.

%% file: model.tex
\section{Model}
\label{sec:model}

We follow the same notation as in \cite{el-mhamdiGoodhartsLawApplication2024} and denote by $G$, $M$ and $\xi$ respectively the intended goal, the proxy metric being optimised, and the discrepancy between the goal and the proxy. 

Namely, an agent is optimising the function $M$ as a proxy for $G$ and $G = M + \xi$. Optimising is modelised in a mechanism agnostic way. Indeed, conditioning on $M>m$ with $m\rightarrow \infty$ model optimisation of the proxy metric.  A basic ``sanity check'' for whether the proxy fails to capture the goal is to see if $M$ and $G$ remain correlated as we optimise $M$. When $M$ and $G$ are not correlated as $M$ is being optimised, one can suspect that optimising $M$ does not help make progress in the intended goal $G$, that is the \emph{weak Goodhart-Strathern} case introduced in the formalism of \cite{el-mhamdiGoodhartsLawApplication2024}. When they keep being correlated, that is the \emph{no Goodhart-Strathern} case. But looking at the correlation between $G$ and $M$, as $M$ is being optimised, is only a first check, one should focus on the intended goal's behaviour. To do so, we also evaluate the expected value of the goal, as the proxy is being optimised, i.e., $\lim_{m \to \overline{\mathfrak{S}_M}} \esp{G | M > m }$ where for any random variable $X$, $\mathfrak{S}_X$ is the support of it and $\overline{\mathfrak{S}}_X = \sup \mathfrak{S}_X$, $\underline{\mathfrak{S}}_X = \inf \mathfrak{S}_X$ the superior and inferior limits of its support respectively. When $\esp{G | M > m }$ decreases as we optimise $M$ that is the same \emph{strong Goodhart-Strathern} case in \cite{el-mhamdiGoodhartsLawApplication2024}. In addition to the \emph{weak, strong} and \emph{no} Goodhart-Strathern, we introduce the \emph{benign Goodhart-Strathern}. In the latter, $M$ stops being correlated with $G$, but $G$ keeps increasing as we optimise~$M$.

\begin{table*}[ht!]
    \centering
    \begin{tabular}{c|c|c}
        &Qualitative definition&Formal definition\\\hline
        No Goodhart & \makecell{During optimisation the proxy stays\\ informative and the goal G goes\\ to its maximum value}&\makecell{$\exists m_0\in \mathfrak{S}(M) / \forall m > m_0,$\\$ \corr{G}{M | M>m} >0 $ and\\ $\esp{G}[M>m] \underset{m \rightarrow \overline{\mathfrak{S}}_M}{\longrightarrow} \overline{\mathfrak{S}}_G $} \\\hline
         Benign & \makecell{Despite the proxy's decreasing \\informativeness, the goal is maximised} &\makecell{$\corr{G}{M | M>m} \underset{m \rightarrow \overline{\mathfrak{S}(M)}}{\longrightarrow} 0 $
         \\ and $\esp{G}[M>m] \underset{m \rightarrow \overline{\mathfrak{S}(M)}}{\longrightarrow} \overline{\mathfrak{S}(G)} $}\\\hline
         Weak & \makecell{The expected value of the goal is \\bounded below its maximum value \\during optimisation} &\makecell{$\exists l \in S(G), l < \overline{\mathfrak{S}(G)}, \exists m_0 \in \mathfrak{S}(M)$ \\
         $/ \forall m > m_0, \esp{G}[M>m] <l$}\\ \hline
         Strong & \makecell{The goal goes to its minimum value \\during the optimisation of the proxy}&$\esp{G}[M>m] \underset{m \rightarrow \overline{\mathfrak{S}(M)}}{\longrightarrow} \underline{\mathfrak{S}(G)}$
    \end{tabular}
    \caption{Qualitative and formal definitions for each of the Goodhart's law outcomes}
    \label{tab:Def}
\end{table*}

%% file: results.tex
\section{Results}
\label{sec:results}

We first provide a comprehensive overview of different cases in  Subsection~\ref{sec:short_results}, before giving our formal results in Subsection~\ref{sec:gaussian_gaussian} and Subsection~\ref{sec:exp_heavy} together with proof sketches, all detailed proofs are available in the Appendix. Also, due to the cumbersome nature of the computation in the last case (exponential goal and heavy tail discrepancy) and to improve reproducibility, we verified the computation with the \emph{symbolic computation} language Sympy~\cite{meurerSymPySymbolicComputing2017} and provide the code to do so in the supplementary materials.

\subsection{Overview} 
\label{sec:short_results}

Precedent work supposed independence between the goal $G$ and the discrepancy $\xi$ in the equation defining the metric $M = G + \xi$. Our work makes two key contributions : 
\begin{enumerate}
    \item  We derive a general result showing the contour of tail thickness importance in Goodhart's law in the heavy tail goal and light tail discrepancy. This results is made with no assumption on dependence structure, and show that in this case there is no possible weak or strong Goodhart possible, as the lower bound we derive goes to infinity with optimisation. However, it keeps open the questions around benign Goodhart in this case and when it happens. 
    \item Then, we capture the dependency between the goal $G$ and the discrepancy $\xi$ in a way that enables a comprehensive analysis, and conduct the analysis in two scenarios:

\paragraph{First scenario (light tailed, Gaussian case).} The goal $G$ and discrepancy $\xi$ form a Gaussian random vector (ie they are both light tailed). Our result extends that of \cite{el-mhamdiGoodhartsLawApplication2024} to the case where the covariance between the goal $G$ and the discrepancy $\xi$ is not null. When maximising the metric (i.e conditioning on $M>m$, with $m \rightarrow \infty$), and provided that we have $\var{G}> \var{\xi}$, we have 3 results :
\begin{itemize}
    \item The true goal $G$ will also be maximised ($\esp{G}[M>m] \underset{m \rightarrow \infty}{\rightarrow} \infty$), although with a coefficient depending on the covariance. 
    \item Despite the conditional expectation of $G$ going to infinity, the correlation between the proxy metric $M$ and the goal $G$ goes to $0$. This is an instance of what we coin the \emph{benign Goodhart's law} (defined intuitively and formally in Table~\ref{tab:Def}). 
    \item The covariance between the goal $G$ and the discrepancy $\xi$ acts linearly on the correlation between the goal $G$ and the proxy metric $M$ when close to zero. When the same covariance is near its limiting value (ie $|\cov{G}{\xi}| \sim \sqrt{\var{G}\var{\xi}}$), the correlation between the goal and the proxy metric can be arbitrarily close to one.
\end{itemize}

\paragraph{Second scenario (heavy tailed discrepancy, exponential goal).} The goal $G$ is exponentially distributed, and the discrepancy $\xi$ is heavy tailed  with conditional law proportional to $\exp(G(\left(x/\eta\right)^{b-1})x^{b-2}$. In this case we have two results : 
\begin{itemize}
    \item We subsume the findings of El-Mhamdi and Hoang \cite{el-mhamdiGoodhartsLawApplication2024} that a heavy tail on the discrepancy makes the conditional expectancy of the goal $G$ goes to 0 when optimising the proxy metric $M$, demonstrating an instance of the strong Goodhart's law.
    \item A bigger shape parameter for the discrepancy $\xi$ (which imply \emph{lighter tail}) will make the goal $G$ goes to $0$ quicker. That is, the same conditioning by $M>m$ will imply a smaller expected value for the goal with a lighter tail discrepancy.
\end{itemize}
\end{enumerate}

\begin{table*}[t]
\centering
\begin{tabular}{c|c|c|c}
&\diagbox{$G$}{$\xi$}&Heavy tail&Light tail\\\hline
\multirow{2}*{\makecell{\textbf{Assuming} \\\textbf{independence} \\\cite{el-mhamdiGoodhartsLawApplication2024},\\ \cite{kwaCatastrophicGoodhartRegularizing2024}}}&Heavy tail&\makecell{With pareto laws on $G$ and $\xi$, \\relevance of the proxy depends on \\the relative tail shape  between \\$G$ and $\xi$ }&\makecell{No Goodhart with $G$ \\having a Pareto law and\\ several light tail or\\ bounded law for $\xi$}\\\cline{2-4}
&Light tail&\makecell{Weak Goodhart worsening with tail\\ \textbf{thickness} of the discrepancy}&Benign Goodhart \\\hline
\multirow{2}*{\makecell{\textbf{No assumption} \\ on independance \\(\textbf{This paper})}}& Heavy tail & x & \makecell{\textbf{No Goodhart} or\\ \textbf{Benign Goodhart}} 
\\\cline{2-4}
&Light tail&\makecell{\textbf{Strong} Goodhart, worsening with\\ tail \textbf{lightness} of the discrepancy\\ in a specific exponential goal to \\heavy tail discrepancy \\dependence structure}& \makecell{\textbf{Benign} Goodhart in the\\Gaussian case}
\end{tabular}
\caption{Summary of results with respect to the goal $G$ and discrepancy $\xi$ tails in state of the art analyses and in our analysis.
}
\end{table*}

\subsection{Heavy tailed goal and light tailed discrepancy}

Here, we consider the general case where the goal is heavy tailed and the discrepancy is light tailed, without specifying any dependence structure nor narrowing the case on particular distributions. 

In this case, we are able to derive a lower bound on the expected value of the true goal $G$ after optimisation, as shown in the following theorem.

\begin{theorem}\label{th:HTLT:No_goodhart}

Suppose that G is heavy tailed, ie denoting $S_G$ the survival function of $G$ we have : $S(G) = \frac{l(t)}{t^{1/\gamma}}$ where $l$ is a slowly varying function (ie $\forall a >0, \underset{t \rightarrow \infty}{\lim}\frac{l(at)}{l(t)} = 1$) and $\xi$ is a right and left light tailed random variable, ie $\exists s \in \R / \frac{\proba{\xi>t}}{e^{-st}} \underset{t \rightarrow \infty}{\rightarrow} 0 \text{ and } \frac{\proba{\xi<t}}{e^{-s|t|}} \underset{t \rightarrow -\infty}{\rightarrow} 0$. Define $M := G+\xi$. Then : 

\[
\esp{G}[M>m]\underset{m \rightarrow \infty}{\geq} m+ o(m)
\]

\begin{proof}[Proof sketch]

The proof relies on two main ideas. 
\begin{enumerate}
    \item First, for any couple of variables $\begin{pmatrix}X \\Y\end{pmatrix}$ where $X$ is light tailed to the right and to the left, the conditional probability of $X$ given $Y$ will almost surely tend to $0$ with a speed almost equivalent to the unconditional probability, that is for any $0< c < s$: 
\[
\frac{\proba{X>t}[Y]}{e^{-ct}} \overset{a.s}{\underset{t \rightarrow \infty}{\rightarrow}} 0
\]

\item Second, using the first idea, we can operate an informed double-cut of the integral \\ $\int_\R g p_G(g) \int_{x \geq m-g}p_{\xi|G}(x) dx dg$.
\begin{enumerate}
    \item We first cut at 0 the integral, to handle separately the positive and negative case
    \item We then make a cut at $m+t_m$, where $t_m = \frac{2log(m)}{\gamma s} - \frac{log(l(m))}{s}$. This allows a good balance of coverage over the support of $\xi$ and a decay of $\proba{G> m+t_m}$ of the order $1/m^{1/\gamma}$.
\end{enumerate}

Figure~\ref{fig:cut} provides a visual guide for the rational of the cut that constitutes the second (and interpretable) part of our proof. The detailed proof is provided in Appendix~\ref{sec:proof_HTLT}.

\end{enumerate}

\begin{figure}[h]
\centering
\begin{tikzpicture}
\usetikzlibrary{decorations.pathreplacing}

\draw[-|,thick] (-3,0)--(-1,0) 
node[pos = 0, left]{$-\infty$} 
node[text width=2.5cm, midway, text width=2.5cm] (first_cut) {}
node[pos = 1, below] {0}
;
\draw [decorate,decoration={brace,amplitude=5pt,raise=2ex}]
 (-3,0) -- (-1,0) {};
\draw [decorate,decoration={brace,amplitude=5pt,mirror, raise=4ex}]
(-1,0) -- (1,0) {};
\draw [decorate,decoration={brace,amplitude=5pt, raise=2ex}]
(1,0) -- (3,0) {}
node[pos = 1, right] {$+\infty$};

\draw[-|,thick] (-1,0)--(1,0)
node[text width=2.5cm, midway, above] (second_cut) {}
node[below, pos = 1] {$m+t_m$};
\draw[->,thick] (1,0)--(3,0)
node[text width=2.5cm, midway, above] (third_cut) {};
\node[rectangle, text width=3.5cm, draw, above] at (-2,0.75){\scriptsize $\proba{\xi >m-G|G}$ is negligible here due to tail lightness of $\xi$ };
\node[rectangle, text width=3.5cm, draw] at (0,-1.5){\scriptsize Light tail on $\xi$ makes the lower bound on this part negligible};
\node[rectangle, text width=3.5cm, draw, above] at (2,0.75){\scriptsize Possible values of $\xi$ cover most of its mass (due to left tail lightness) while $G$ probability decay is $\sim \frac{1}{m^{1/\gamma}}$};
\end{tikzpicture}
\caption{Visualisation of the rationale behind the cut that we operate on the integral to enable the proof.}
\label{fig:cut}
\end{figure}
\end{proof}

\end{theorem}

This theorem shows that weak and strong Goodhart are impossible in this case.
Indeed, the lower bound show that the expected value of $G$ after optimisation will go to infinity.
This result is fully independent from the dependence structure between the true goal $G$ and the discrepancy $\xi$, the relative tail thickness allowing to counter-act any unfavorable dependence structure.
This is why we need both left and right tail of $\xi$ to be exponentially decreasing for this results to hold.
Otherwise we can construct a dependence structure such that the left tail of the discrepancy (where it's negative) completely compensate for the positive value of the true goal $G$.
This emerges in the proof through the trade-off between the coverage of the support of $\xi$ and decreasing speed of $\proba{G>m+t_m}$ when choosing $t_m$ : with a left tail too thick the trade-off becomes impossible.

If this result prevents weak and strong Goodhart, it does not tell anything about the proxy informativeness with respect to the true goal $G$. As such, it's still an open question to know when a benign Goodhart rather than "no Goodhart" will happen.

\subsection{Light tailed goal and light tailed discrepancy: the Gaussian case}
\label{sec:gaussian_gaussian}
In this section, we study the double light-tailed situation where both $G$ and $\xi$ are Gaussian. For notation convenience,  we represent $G$ and $\xi$ as a Gaussian random vector as follows. 
\[
\left(\begin{matrix}
G \\
\xi
\end{matrix}\right)\sim \mathcal{N}(0,\Sigma)
\]
where $\Sigma = \left(\begin{matrix}
a &c\\
c &b \\
\end{matrix}\right)$ is the covariance matrix of the random vector composed by $G$ and $\xi$. Here, $\var{G} = a$, $\var{\xi} = b$ and $\cov{G}{\xi} = c$. The first result shows that in the Gaussian case, as long as the variance of the goal $G$ dominates the variance of the discrepancy $\xi$ ($\var{G}>\var{\xi}$), the goal $G$ goes to infinity while the correlation between the goal $G$ and the proxy metric $M$ goes to 0. We call this situation the \emph{benign Goodhart's law}, which is a special case of the ``weak Goodhart'' case introduced in \cite{el-mhamdiGoodhartsLawApplication2024}. In this case, ``benign`` reflects the fact that while the correlation between the goal and the proxy is going to zero it does not prevent the goal from going to infinity. 
\begin{lemma}\label{lma:esp_G_normal}
With $(G, \xi) \sim \mathcal{N}(0,\Sigma)$, $M = G + \xi$ and $\var{G}>\var{\xi}$ the optimisation of the proxy metric also leads to the optimisation of the true goal.

\[\esp{G}[M>m]\underset{m \rightarrow \infty}{\sim} \frac{a+c}{(a+b+2c)}m.\]
\end{lemma}

The full proof of Lemma~\ref{lma:esp_G_normal} is given in Appendix~\ref{lma:NN:5}, bellow we provide a simple proof sketch.

\begin{proof}[Proof sketch]
The proof contains two parts. The first part computes an equivalent for $\proba{M>m}$, which is done using the equivalent for Gaussian tail :
\[\int_x^{\infty}e^{-u^2}du = \frac{e^{-x^2}}{2x}\sum_{n=0}^{N-1}(-1)^n \frac{(2n-1)!!}{2^nx^{2n}}+O\left(\frac{x^{-2N-1}}{\exp(x^2)}\right).\]

The second part computes the unormalised expected value from the  formula of its definition, i.e., 

\[\int_\R\exp(- \sigma_2x^2)\int_{t \geq \mu - x\theta}t \exp\left(- \sigma_1 t^2\right) dt dx.\]

\end{proof}

The coefficient in front of the optimisation threshold $m$ is most of the time $<1$, except if the covariance between the goal $G$ and the discrepancy $\xi$ is close to its minimum value. With greater positive covariance, the discrepancy $\xi$ will account for a greater portion of the proxy, thus decreasing the expected value of the goal $G$. On the contrary, with negative covariance between the goal $G$ and $\xi$, an increase in the goal value induce a decrease proportional to the covariance in the discrepancy as we have $\esp[G]{\xi} = \frac{cG}{a} $. This leads to the goal being actually higher in expectancy with negative covariance, as for any level of the proxy considered, it has to compensate for the discrepancy that is negatively correlated.

We coin the term ``benign'' to describe the situation as despite the goal $G$ going to infinity, the correlation between the proxy $M$ and the goal $G$ goes to 0

\begin{theorem}
\label{th:corr_normal}
With $(G, \xi) \sim \mathcal{N}(0,\Sigma)$, $M = G + \xi$, the correlation between the proxy metric $M$ and the goal $G$ goes to zero in the limit no matter the correlation between the discrepancy $\xi$ and the goal $G$.
\[
\corr{G}{M | M>m} \underset{m \rightarrow \infty}{\sim} \frac{(a+c)\sqrt{a+b+2c}}{m\sqrt{ab-c^2}}.
\]
\end{theorem}

The full proof of Theorem~\ref{th:corr_normal} is given in Appendix~\ref{proof:corr_normal}, bellow we provide a simple proof sketch.

\begin{proof}[Proof sketch]
The proof uses the formula for conditional variance an covariance (taking $X$,$Y$ and $Z$ random variables) : 
\begin{align*}
\var{X}[Y] =& \esp{X^2}[Y] + \esp{X}[Y]^2 ,\\
\cov{X}{Y}[Z] =& \esp{XY}[Z] + \esp{X}[Z]\esp{Y}[Z].
\end{align*}

For the squared term, a computation of the conditional leads to a first tractable term, and a second one intractable. The second intractable term is the Gaussian density integrated over the half space in $\R^2$ delimited by $\theta x+y > \mu$, i.e.,
 \[
 \mathfrak{A} = \frac{1}{\sigma_2} \int_\R \int^{+\infty}_{m - x\theta}\exp(-\frac{\sigma_1}{2}x^2 - \frac{\sigma_2}{2}t^2) dt dx .
 \]

$\mathfrak{A}$ is proportionnal to $\proba{X+Y>m}$ where $X$ and $Y$ are independant Gaussian random variables of variance $\frac{1}{\sigma_1}$ and $\frac{1}{\sigma_2}$. But we know that the sum of $X$ and $Y$ is a Gaussian random variable of variance $\frac{1}{\sigma_1} + \frac{1}{\sigma_2}$. As such, an equivalent for the tail of $Z := X+Y$ is also an equivalent for $\mathfrak{A}$. We compute it with the equivalent for Gaussian tail.

The crossed term poses no difficulty. The simple expectation is calculated as in the proof for $G$ expectation.
\end{proof}

Here , we can differentiate two regimes of the covariance on the correlation equivalent, exemplified in Figure \ref{fig:coefficient_correlation} : 
\begin{itemize}
    \item \textbf{When the covariance coefficient is close to 0}, moving the covariance will move almost linearly the correlation coefficient.
    \item \textbf{When the covariance coefficient is close to the limit}, the covariance matrix is almost degenerated and the point of the Gaussian lie very close to a line. It means $M \approx const\times G$. One of the consequence is that for any quantile we can choose, there exist a Gaussian coupling such that the correlation is arbitrarily close to 1 in that quantile.
\end{itemize}

\begin{figure}[ht]
    \centering
    \includegraphics[width=0.40\textwidth]{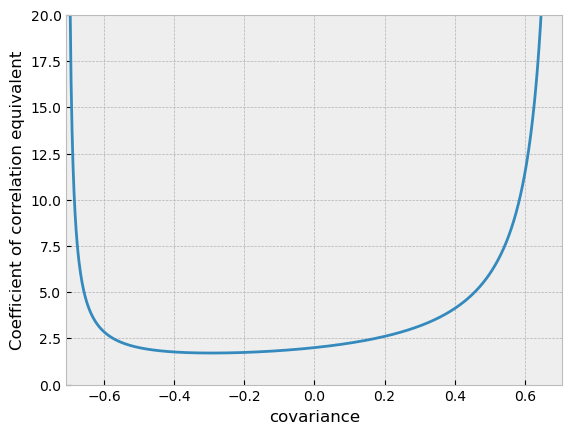} 
    \caption{Coefficient for the correlation equivalent depending on the value of the covariance}
    \label{fig:coefficient_correlation}
\end{figure}

\begin{table}[ht]
\begin{center}

\begin{tabular}{c|c|c}
    & $c < 0$& $c >0 $\\\hline
    $\esp{G}[M>m]$ & + & - \\
    $\corr{G}{M | M>m}$ & - & + 
\end{tabular}
\caption{Normal goal and discrepancy results summary}\label{tab:summary}
\end{center}
\end{table}

The effect of covariance on conditional expectancy of the goal $G$ and correlation between the proxy metric $M$ and goal $G$ is contrasted. Table \ref{tab:summary} summarizes the covariance effect, ``+'' denoting an improvement on the considered quantity while ``-'' denotes a negative effect.

\subsection{Exponential goal and heavy tailed discrepancy}
\label{sec:exp_heavy}
In this section, we consider the case where the goal $G$ is exponential of parameter $1$. The discrepancy is then drawn conditionally to the goal $G$ by a truncated at $1$ exponential law of parameter $G$. This imply that $\xi$ follow a Pareto law of shape $2$ and scale $1$. The same variable elevated to the power of $1/(b-1)$ and multiplied by $\eta$ will be such that it follow a pareto law of shape $b$ and scale $\eta$.

$\xi$ then has the conditionnal density : $f_{\xi\mid G}(x) = G\exp(-G((\frac{x}{\eta})^{b-1}-1))\frac{x^{b-2}}{\eta^{b-1}}(b-1)\1[x > \eta]$.
This case is an example of strong Goodhart's law. The optimisation the goal $G$ tends to 0 while making the discrepancy tend to infinity.

\begin{lemma}\label{lma:esp_noise_heavy}
When $G\sim\mathcal{E}(1)$ and $f_{\xi\mid G}(x) = G\exp(-G((\frac{x}{\eta})^{b-1}-1))\frac{x^{b-2}}{\eta^{b-1}}(b-1)\1[x > \eta]$, the maximisation of the proxy metric $M$ also leads to the maximisation of the discrepancy $\xi$, as
\[
\esp{\xi}[M>m] \underset{m\rightarrow\infty}{\sim} m\frac{b - 1}{b - 2}.
\]
\end{lemma}

The full proof can be found in Appendix \ref{lma:EHT:10}

\begin{proof}[Proof sketch]
The idea of the proof is inspired by Gaussian tail development. Using the fact that for any polynomial $Q$, we have $\frac{\mathrm{d}\exp{Q(x)}}{\mathrm{d}x} = Q'(x)\exp(Q(x))$, this means that any integral of the form
$\int_0^{m-\eta} P(x)exp(Q(m,x))dx$ can be iteratively integrated by part to obtain as follows.
 \begin{align*}
 &\int_0^{m-\eta} P(x)exp(Q(m,x))dx   \\
 = &\underbrace{\sum_{n=0}^{N}\left[ \frac{f_{Q'}^{n}(P(g))}{Q'(g)}\exp(Q(g)) \right]_0^{\frac{m}{b+1}}}_{I_1} +(-1)^{N+1} \underbrace{\int_0^{m/b+1}f_{Q'}^{N+1}(P)(g)\exp(Q(g))dg}_{I_2},
 \end{align*}
 
where we denote $f_{Q'}$ the operation consisting in dividing by $Q'$ and then differentiating, $f_{Q'}^n$ consisting in applying $f_{Q'}$ n times. Using the fact that $f_{Q'}^{N+1}(P)$ is bounded (as it is a rational fraction with degree $< 0$) by a quantity decreasing to zero with a speed depending on $N+1$, $I_2$ is $o(I_1)$ (using the Bachmann–Landau notations to reflect the fact that $I_2$ is of inferior order to $I_1$).
The expected values are then an application of this equivalent after a some work on the original integral.

\end{proof}

First important thing to notice is that the discrepancy $\xi$ is actually what's being optimised for here. The lighter tail of the goal $G$ makes it much less likely than the discrepancy $\xi$ to produce extreme realisation. This is true whatever the shape $b$ of the discrepancy $\xi$ as its tail is proportional to $1/x^b$, while the tail of the exponential goal $G$ is proportional to $exp(-x)$. This means that, if we know that $M$ has a very high realisation, it will be much more likely to be due to a large discrepancy $\xi$ than a high $G$.

Moreover, as the discrepancy is drawn approximately following an exponential distribution of parameter $G$ (the goal), this means that high realisation of the discrepancy $\xi$ are associated with small realisation of the goal $G$. As such, the optimisation procedure - by increasing the likelihood of higher discrepancy $\xi$ - will also make instances of very small $G$ much more likely. This leads to following result.

\begin{lemma}\label{lma:esp_goal_heavy}

When $G\sim\mathcal{E}(1)$ and $f_{\xi\mid G}(x) = G\exp(-G((\frac{x}{\eta})^{b-1}-1))\frac{x^{b-2}}{\eta^{b-1}}(b-1)\1[x > \eta]$, the goal is minimised when the proxy metric is maximised, as 
\[
\esp{G}[M > m] \underset{m\rightarrow\infty}{\sim}\frac{\eta^{b - 1}}{m^{b-1}}.
\]
\end{lemma}

\begin{proof}[Proof sketch]
Similarly to Lemma~\ref{lma:esp_noise_heavy}, we operate a tail development as a starting point for the derivation, of which the full details are provided in Appendix~\ref{lma:esp_goal_heavy}.  
\end{proof}

Lemma~\ref{lma:esp_goal_heavy} also shows that the lighter the tail of the discrepancy is (ie the bigger the shape $b$ of $\xi$ is), the faster the goal will decrease toward 0.
This is because a light tail on the discrepancy $\xi$ will mean more probability mass near $\eta$, which will be associated with bigger value of the goal $G$ first.
Moreover, Lemma~\ref{lma:esp_goal_heavy}  means that for any high value of the discrepancy $\xi$, it's realisation will be associated to a smaller $G$ in expectancy if the discrepancy $\xi$ as a lighter tail.

The two precedent results shows the importance of the coupling when talking about Goodhart's law. Indeed, if the finding here is in line with \cite{el-mhamdiGoodhartsLawApplication2024} with heavy tailed discrepancy, it brings nuance in the fact that the coupling here is such that the goal decrease at a speed which is actually inverse to that of the tail heaviness.

%% file: discussion.tex
\section{Discussion}
\label{sec:discussion}

Whether Goodhart's law occurs or not has been recently shown to depend on the tail thickness of both the intended goal and the discrepancy between that goal and the actual proxy metric. This has been done~\cite{el-mhamdiGoodhartsLawApplication2024, kwaCatastrophicGoodhartRegularizing2024} while assuming the two aforementioned quantities to be independent. Our results depict a complex image of the relationship between tail thickness and dependence structure when it comes to Goodhart's law as follows. 

We prove that the heavy tailed goal - light tailed discrepancy weak and strong Goodhart are not possible. However, the two other cases we analyzed, - namely, the Gaussian case and the exponential goal - heavy tailed discrepancy,- show the importance of dependence structure when the relative tail thickness are not favorable or of the same kind. Indeed, in the heavy tailed discrepancy and exponential goal case, the dependence structure transforms a weak Goodhart case (in the independent case) into a strong Goodhart. Moreover, this strong Goodhart case gets worse as the the tail of the discrepancy gets lighter. This, with our precedent result, is to our sense a hint to the fact that tail thickness is not the \emph{trigger}, but rather something that will \emph{enable} the presence of Goodhart's law. An unbalanced tail thickness of the true goal and the discrepancy enables some dependence structure to trigger a Goodhart's law situation when optimizing.

Given our results, we consider that a particular attention should be taken in situations where the discrepancy and the true goal might not be independent. As an example, many industrially deployed algorithms suffer
from their evaluation on historic data, while this data is itself biased by the actual policy used to produce it and resulting in a feedback loop and distribution shift~\cite{taoriDataFeedbackLoops2022, levineOfflineReinforcementLearning2020}.

Figure~\ref{figure:application_example} is an example of such a bias in data resulting in a feedback loop. We denote by $\pi$ the possible policy being learned, and by $M(\pi)$ the metric (build with the biased dataset) evaluated for $\pi$. We dashed the link between the policies $\pi_1$, $\pi_2$ and the biased dataset as our setup is purely optimisation agnostic. 

Assuming we have access to a dataset of true goal measurement~\footnote{This is not possible in all settings, but in any settings where the true goal might be computable but too expensive to train on it directly. This is also the case in settings where the true goal or a \textbf{much better} approximation of it is accessible during deployment. While creating a new research question, related to sampling error, this situation remains more informative on the true goal than the general agnostic case.}, our work may be used to assess the error from Goodhart's law through the whole process of policy optimisation, from training to potential A/B testing. We refer to this whole process through the $argmax$ operator on a set of policies being tested $\Pi$ with the metric $M$. To keep the diagram simple, we illustrate the selection between two policies, but the setting is applicable to any selection procedure. The selected policy (denoted by $\pi_s$) will then be deployed at scale, allowing to evaluate the true goal for this policy, denoted $G(\pi_s)$, but also biasing the dataset on which future evaluation will be made. Any learning trade-off would result in a dependency between the metric and the true goal discrepancy in at least a similar fashion to what follows.  The learning trade-off could lead to a worse prediction on a small part of the population but to a better overall prediction might be poorly evaluated is the precedent selected policy over-sampled the small part of the population on which the new policy perform less, thus making a \emph{better policy} more likely to have a \emph{worse metric}.

Using our work and a dataset of true goal measurement, as has been just described, we could typically perform a tail and copula estimation (a well known problem in insurance \cite{caillaultEmpiricalEstimationTail2005,chorosCopulaEstimation2010,derumignyConditionalEmpiricalCopula2020}), predict failure points of the metric and devise protective measures. Other empirical evaluations are discussed in the ``Empirical Evaluation'' paragraph bellow.

\begin{figure}
    \centering
\begin{tikzpicture}

\usetikzlibrary{shapes.geometric, arrows.meta, shapes.misc}

\node[rectangle, align = center, text width=8ex, draw] (b_dataset) at (-3,0) {Biased dataset};
\node[star, align = center, text width=1ex, label=center:$\pi_1$, draw] (pi_1) at (-1.5,1) {}; 
\node[star, align = center, text width=1ex, label=center:$\pi_2$, draw] (pi_2) at (-1.5,-1) {}; 

\draw [-Stealth, dashed] (b_dataset.north east) |- (pi_1.west);
\draw [-Stealth, dashed] (b_dataset.south east) |- (pi_2.west);

\node[regular polygon, regular polygon sides=7, align = right, text width=4ex, label=center:$M(\pi_1)$, draw]
(m_pi1) at (0.5,1) {};
\node[regular polygon, regular polygon sides=7, align = center, text width=4ex,label=center:$M(\pi_2)$, draw]
(m_pi2) at (0.5,-1) {};

\draw[-Stealth] (b_dataset.east) -| (m_pi1);
\draw[-Stealth] (b_dataset.east) -| (m_pi2);

\draw[-Stealth] (pi_1) -- (m_pi1);
\draw[-Stealth] (pi_2) -- (m_pi2);

\node[] (max) at (2.5,0) {$\underset{\pi \in \Pi}{argmax} M(\pi) =: \pi_s$};

\draw [-Stealth] (m_pi1.east) -| (max.north);
\draw [-Stealth] (m_pi2.east) -| (max.south);

\node[rectangle, draw] (true) at (0,-2) {Deployement};

\draw[-Stealth] (max.-20) |- (true.east)
node[pos = 0.75, below] {$\pi_{s}$};

\node[circle, draw, text width = 5ex, label = center:$G(\pi_{s})$] (goal) at (0,-3) {};

\draw[-Stealth] (true)-|(b_dataset)
node[pos = 0.25, above] {sample with}
node[pos = 0.25, below] {$\pi_s$};
\draw[-Stealth] (true.south)--(goal.north);

\end{tikzpicture}
\caption{Example of bias scheme and the feedback look that result from it.}
\label{figure:application_example}
\end{figure}

We see several natural continuations to our work, and list some of the in the following paragraphs.

\paragraph{Empirical evaluation. } Our theoretical work highlights the prominent role played by the tail thickness in deciding different outcomes in Goodhart's law and provides a nuanced picture of these outcomes. Empirical follow-ups of our work could bring even more nuance by performing fine-grained studies on such outcomes. 

Notably, several experimental designs are possible in settings where we have access to the true goal and thus can evaluate the discrepancy
\begin{itemize}
    \item In supervised learning, with voluntarily biased dataset (by flipping a certain number of label depending on certain characteristics for example), to assess how the bias influence the metric's tail and dependence structure with the goal (possibly making an bridge with fairness research).
    \item In reinforcement learning, assessing the presence of heavy tailed reward ``in the wild''~\cite{clarkFaultyRewardFunctions2016} might be a good indicator of bad behavior possibility.
    \item Testing the tail behavior of learned reward model with respect to gold standard (in a similar fashion as in \cite{gaoScalingLawsReward2022}) is also to our sense an interesting direction to empirically assess the soundness of this research.
\end{itemize}

\paragraph{Access to the proxy metric's tail. } As proxy metrics' tail seems to determine the absence, presence and strength of Goodhart's law, empirical study on metrics' tail thickness within real world applications would be key to assess the importance of Goodhart's law. Devising empirically grounded categories of tasks that are subject to heavy tail losses could offer a needed roadmap for practitioners to avoid useless or harmful metrics optimization or at least be aware of potential risky situations.

\paragraph{Aggregation of metrics. } In real world settings, we have access to neither the goal nor the discrepancy. However, we might have access to several proxy metrics $M_1, M_2, \dots$ representative of the same overall goal $G$. Using the multiplicity of proxy metrics at disposal might be key to alleviate Goodhart's law, notably by devising aggregation rules that would make an aggregated proxy metrics $\Tilde{M}$ more robust through aggregation and natural variance reduction. This would be to our sense key to alleviate alignment problem in concrete AI implementation.

\paragraph{Goodhart's law and evasion attacks.} \cite{hennessyGoodhartsLawMachine2020} study Goodhart's law within evasion attack settings\footnote{Often called adversarial attacks, we prefer the more specific term \emph{evasion attacks} to distinguish them from other adversarial attacks~\cite{vassilevAdversarialMachineLearning2024} such as poisoning attacks and data-extraction/privacy attacks.}. So far we only considered settings where no malicious or adversarial player were~present. In practical settings, AI faces adversarial behaviour~\cite{linUsingAdversarialAttacks2021, grosseMachineLearningSecurity2023} that must be anticipated to avoid catastrophic failure. Creating concrete threat models as well as defense mechanisms represent a large enough scope for development well beyond this paper. But doing so would benefit from the first step that was done in our work, in separating different the possible outcomes in the absence of an adversary, before adding more complexity and studying the interaction between an adversary and the inherent issues related to Goodhart's law.

\paragraph{Auditing.} Auditing black box models is hard \cite{godinotManipulationsAreAI2024}. Peculiarly, developing robust and non-hackable metrics is of prime interest when auditing ML models. As such, understanding Goodhart's law~can inform and strengthen research on auditing ML models. Our work offers a starting point in formally categorizing what auditors should look for, namely, the different strengths of Goodhart's law and the different statistical features leading to them.

\paragraph{Theoretical guarantees.} Theoretical guarantees on the possibility of harmful behaviour by AI prior to or at test~time are key to mitigate the global risk of AI. Moreover, probabilistic framing of such guarantees is the subject of repeated calls to inform AI safety research~\cite{bengioInternationalAISafety2025}.

While providing such desired probabilistic formulations, our paper highlights the collapse of any prospect for alignment in the presence of heavy-tails and does so with much more nuance than has been previously done with additional assumptions. As such, this work contributes to a better theoretical understanding of AI misalignment, and can serve as the basis for a Goodhart-centered research agenda on AI alignment based on the list of possible continuations discussed in this section.

%% file: ackowledgment.tex
\subsection*{Acknowledgment}

Funded by the European Union (ERC-2022-SyG, 101071601). Views and opinions expressed are however those of the author(s) only and do not necessarily reflect those of the European Union or the European Research Council Executive Agency. Neither the European Union nor the granting authority can be held responsible for them.

%% file: appendix.tex
\section{Appendix}

\subsection{Normal goal and  normal discrepancy :}

\subsubsection{Lemmas :}

\begin{lemma}\label{lma:NN:1}
    $\displaystyle{\int_{m}^{+\infty} \exp(-x^2)dx \underset{m\rightarrow \infty}{=} \frac{\exp\left(-m^2\right)}{2m}\left[1 - \frac{1}{m^2} + \frac{1}{2m^4} + o(m^{-5}) \right]}$.
\end{lemma}

\begin{proof}
    This directly stems from the well known equivalent for the Gaussian tail : \[\int_x^{\infty}e^{-u^2}du = \frac{e^{-x^2}}{2x}\sum_{n=0}^{N-1}(-1)^n \frac{(2n-1)!!}{2^nx^{2n}}+O(x^{-2N-1}\exp(-x^2)).\]
    Applying it for $N = 2$ yields the lemma
\end{proof}

\begin{lemma}\label{lma:NN:2}
For $\theta > 0$, $\sigma_1 > 0$ and $\sigma_2 >0$ :
\[\int_\R\exp\left(-\frac{x^2}{2\sigma_1}\right)\int_{t \geq m - x\theta}tx\exp\left(-\frac{t^2}{2\sigma_2}\right)  dt dx =\sqrt{2\pi}\frac{2c(b+c)}{b^2}\left(\frac{ab-c^2}{a+b+2c}\right)^{3/2}m_\alpha\exp\left(-\frac{m_\alpha^2}{2(a+b+2c)} \right).\]
\end{lemma}

\begin{proof}
\begin{align*}
    \mathrm{I}_1 &= \int_\R x \exp\left(-\frac{x^2}{2\sigma_1}\right)\int_{t \geq m - x\theta}t\exp\left(-\frac{t^2}{2\sigma_2}\right)  dt dx \\
    &= \int_\R x \exp\left(-\frac{x^2}{2\sigma_1}\right)\left[-\sigma_2\exp\left(-\frac{t^2}{2\sigma_2}\right)\right]^{+\infty}_{m - x\theta} dx \\
    &= \sigma_2 \int_\R x \exp\left(-\frac{x^2}{2\sigma_1}\right)\exp\left(-\frac{\left[m_\alpha - x\theta\right]^2}{2\sigma_2}\right)dx \\
    &= \sigma_2 \int_\R x \exp\left(-\frac{1}{2}\left(x^2\left(\frac{1}{\sigma_1} + \frac{\theta^2}{\sigma_2}\right) + \frac{m_\alpha^2}{\sigma_2} - 2\frac{m_\alpha x}{\sigma_2}\right)\right)dx \\
    &= \sigma_2 \exp\left(-\frac{m_\alpha^2}{2\sigma_2}\right) \int_\R x \exp\left(-\frac{1}{2}\left(x^2\left(\frac{\sigma_2 + \theta^2}{\sigma_2\sigma_1}\right) - 2\frac{m_\alpha x}{\sigma_2}\right)\right)dx \\
    &= \sigma_2 \exp\left(-\frac{m_\alpha^2}{2\sigma_2}\right) \int_\R x \exp\left(-\frac{\sigma_2 + \theta^2}{2\sigma_2\sigma_1}\left(x -\frac{\sigma_1}{\sigma_2 + \theta^2}m_\alpha\right)^2 + \frac{\sigma_1}{(\sigma_2 + \theta^2)\sigma_2}m_\alpha^2\right)dx \\
    &= \sigma_2 \exp\left(-\frac{(\sigma_1 + \sigma_2 +\theta^2)m_\alpha^2}{2\sigma_2} \right) \int_\R x \exp\left(\frac{\sigma_2 + \theta^2}{2\sigma_2\sigma_1}\left(x -\frac{\sigma_1}{\sigma_2 + \theta^2}m_\alpha\right)^2\right)dx .\\
    \intertext{With the change of variable $u = x -\frac{\sigma_1}{\sigma_2 + \theta^2}m_\alpha$,}
    &= \sigma_2 \exp\left(-\frac{(\sigma_1 + \sigma_2 +\theta^2)m_\alpha^2}{2\sigma_2} \right) \left[\underbrace{\int_\R u \exp\left(-\frac{\sigma_2 + \theta^2}{2\sigma_2\sigma_1}u^2\right)du}_{=0} + \frac{\sigma_1}{\sigma_2 + \theta^2}m_\alpha\int_{\R} \exp\left(-\frac{\sigma_2 + \theta^2}{2\sigma_2\sigma_1}u^2\right) du\right] \\
    &= \frac{\sigma_2\sigma_1}{\sigma_2 + \theta^2} m_\alpha\exp\left(-\frac{(\sigma_1 + \sigma_2 +\theta^2)}{2\sigma_2\sigma_1}m_\alpha^2 \right)\sqrt{\frac{2\pi\sigma_1\sigma_2}{\sigma_2+\theta^2}} \\
    &= \sqrt{2\pi}\left(\frac{\sigma_2\sigma_1}{\sigma_2 + \theta^2}\right)^{3/2}m_\alpha\exp\left(-\frac{(\sigma_1 + \sigma_2 +\theta^2)}{2\sigma_2\sigma_1}m_\alpha^2 \right).
\end{align*}
\end{proof}

\begin{lemma}\label{lma:NN:3}
For $\theta > 0$, $\sigma_1 > 0$ and $\sigma_2 >0$:
    \begin{align*}
        &\int_\R\exp\left(-\frac{\sigma_1}{2}x^2\right)\int_{t \geq \mu - x\theta}t^2\exp\left(-\frac{\sigma_2}{2}t^2\right)  dt dx \\
        \underset{m_\alpha \rightarrow \infty}{=}&\frac{\sqrt{2\pi}}{\sqrt{\sigma_1 + \theta^2\sigma_2}}\exp\left(\frac{-\mu^2\sigma_2\sigma_1}{2(\sigma_1 + \theta^2\sigma_2)}\right)\\
        &\times\left(\frac{\mu\sigma_1}{\sigma_2(\sigma_1 + \theta^2\sigma_2)} + \frac{(\sigma_2\theta^2 + \sigma_1)}{\sigma_1\sigma_2^2\mu}\left(1 - \frac{(\sigma_2\theta^2 + \sigma_1)}{\mu^2\sigma_1\sigma_2} + \frac{3(\sigma_2\theta^2 + \sigma_1)^2}{\mu^4\sigma_1^2\sigma_2^2} + o(\mu^{-5})\right) \right).
    \end{align*}
\end{lemma}
\begin{proof}

First we denote : 
\[\mathfrak{I} := \int_\R\exp\left(-\frac{\sigma_1}{2}x^2\right)\underbrace{\int_{t \geq \mu - x\theta}t^2\exp\left(-\frac{\sigma_2}{2}t^2\right)  dt}_{\mathfrak{I_1}} dx.\]

We integrate by part $\mathfrak{I_1}$ :
\begin{align*}
    \mathfrak{I_1} = \int_{t \geq \mu - x\theta}t^2\exp\left(-\frac{\sigma_2}{2}t^2\right)  dt &= \left[ -\frac{t}{\sigma_2}\exp(-\sigma_2t^2)\right]^{\infty}_{\mu - x\theta} + \frac{1}{\sigma_2}\int_{t \geq \mu - x\theta} \exp(\frac{\sigma_2}{2}t^2) \\
    &= \frac{\mu - x \theta}{\sigma_2}\exp(\frac{\sigma_2}{2}(\mu - x \theta)^2) + \frac{1}{\sigma_2}\int_{t \geq \mu - x\theta} \exp(-\frac{\sigma_2}{2}t^2)dt.
\end{align*}

Pluging it into $\mathfrak{I}$ we get : 
\[
    \mathfrak{I} = \underbrace{\int_\R \frac{\mu - x\theta}{\sigma_2} \exp(\frac{\sigma_2}{2}(\mu^2 - 2\mu\theta x + \theta^2x^2) -\frac{\sigma_1}{2}x^2) dx}_{\mathfrak{I}_2} + \underbrace{\frac{1}{\sigma_2} \int_\R \int^{+\infty}_{\mu - x\theta}\exp(\frac{\sigma_1}{2}x^2 - \frac{\sigma_2}{2}t^2) dt dx}_{\mathfrak{A}}.
\]
$\mathfrak{I}_2$ gives : 
\begin{align*}
    &\int_\R \frac{\mu - x\theta}{\sigma_2} \exp\left(-\frac{\sigma_2}{2}(\mu^2 - 2\mu\theta x + \theta^2x^2) -\frac{\sigma_1}{2}x^2\right) dx \\
    =&\int_\R \frac{\mu - x\theta}{\sigma_2} \exp\left(\frac{-\mu^2\sigma_2}{2}+2\frac{\mu\sigma_2}{2}\theta x - (\frac{\sigma_1 
    + \theta^2\sigma_2}{2})x^2\right) dx \\
    =&\exp(\frac{-\mu^2\sigma_2}{2})\int_\R \frac{\mu - x\theta}{\sigma_2} \exp\left(-\frac{\sigma_1 
    + \theta^2\sigma_2}{2}(x^2 - 2\frac{\mu\sigma_2}{(\sigma_1 + \theta^2\sigma_2)}\theta x + \frac{\mu^2\sigma_2^2\theta^2}{(\sigma_1 + \theta^2\sigma_2)^2}) + \frac{\mu^2\sigma_2^2\theta^2}{2(\sigma_1 + \theta^2\sigma_2)}\right) dx \\
    =&\exp\left(\frac{-\mu^2\sigma_2\left(1 - \frac{\sigma_2\theta^2}{\sigma_1 + \theta^2\sigma_2}\right)}{2}\right)\int_\R \frac{\mu - x\theta}{\sigma_2} \exp\left(-\frac{\sigma_1 
    + \theta^2\sigma_2}{2}(x - \frac{\mu\sigma_2\theta}{(\sigma_1 + \theta^2\sigma_2)})^2\right) dx \\
    =&\exp\left(\frac{-\mu^2\sigma_2\sigma_1}{2(\sigma_1 + \theta^2\sigma_2)}\right)\left(\int_\R \frac{\mu}{\sigma_2} \exp\left(-\frac{\sigma_1 
    + \theta^2\sigma_2}{2}(x - \frac{\mu\sigma_2\theta}{(\sigma_1 + \theta^2\sigma_2)})^2\right) dx\right. \\
    &\left.-\int_\R\frac{x\theta}{\sigma_2} \exp\left(-\frac{\sigma_1 
    + \theta^2\sigma_2}{2}(x - \frac{\mu\sigma_2\theta}{(\sigma_1 + \theta^2\sigma_2)})^2\right) dx\right) .\\
\end{align*}
Making the change of variable $u = x - \frac{\mu\sigma_2\theta}{(\sigma_1 + \theta^2\sigma_2)}$,
\begin{align*}
    &\int_\R \frac{\mu}{\sigma_2} \exp\left(-\frac{\sigma_1 
    + \theta^2\sigma_2}{2}(x - \frac{\mu\sigma_2\theta}{(\sigma_1 + \theta^2\sigma_2)})^2\right) dx -\int_\R\frac{x\theta}{\sigma_2} \exp\left(-\frac{\sigma_1 
    + \theta^2\sigma_2}{2}(x - \frac{\mu\sigma_2\theta}{(\sigma_1 + \theta^2\sigma_2)})^2\right) dx \\
    &= \frac{\mu}{\sigma_2}\int_\R \exp\left(-\frac{\sigma_1 
    + \theta^2\sigma_2}{2}u^2\right) dx -\int_\R\frac{(u + \frac{\mu\sigma_2\theta}{(\sigma_1 + \theta^2\sigma_2)})\theta}{\sigma_2} \exp\left(-\frac{\sigma_1 + \theta^2\sigma_2}{2}u^2\right) dx \\
    &=\frac{\mu}{\sigma_2}\sqrt{\frac{2\pi}{\sigma_1 + \theta^2\sigma_2}} -\underbrace{\int_\R\frac{u\theta}{\sigma_2} \exp\left(-\frac{\sigma_1 + \theta^2\sigma_2}{2}u^2\right) dx}_{=0} - \frac{\mu\sigma_2\theta^2}{\sigma_2(\sigma_1 + \theta^2\sigma_2)}\int_\R\exp\left(-\frac{\sigma_1 + \theta^2\sigma_2}{2}u^2\right) \\
    &=\frac{\mu}{\sigma_2}\sqrt{\frac{2\pi}{\sigma_1 + \theta^2\sigma_2}} -\frac{\mu\sigma_2\theta^2}{\sigma_2(\sigma_1 + \theta^2\sigma_2)}\sqrt{\frac{2\pi}{\sigma_1 
    + \theta^2\sigma_2}}.
\end{align*}

It yields for $\mathfrak{I}_2$ : 
\begin{align*}
    &\exp\left(\frac{-\mu^2\sigma_2\sigma_1}{2(\sigma_1 + \theta^2\sigma_2)}\right)\sqrt{\frac{2\pi}{\sigma_1 
    + \theta^2\sigma_2}}\left(\frac{\mu}{\sigma_2} - \frac{\mu\sigma_2\theta^2}{\sigma_2(\sigma_1 + \theta^2\sigma_2)}\right) \\
    =& \exp\left(\frac{-\mu^2\sigma_2\sigma_1}{2(\sigma_1 + \theta^2\sigma_2)}\right)\sqrt{\frac{2\pi}{\sigma_1 
    + \theta^2\sigma_2}}\frac{\mu\sigma_1}{\sigma_2(\sigma_1 + \theta^2\sigma_2)}. \\
\end{align*}

 We want now to approximate :
 \[
 \mathfrak{A} = \frac{1}{\sigma_2} \int_\R \int^{+\infty}_{\mu - x\theta}\exp(-\frac{\sigma_1}{2}x^2 - \frac{\sigma_2}{2}t^2) dt dx.
 \]

For $\mathfrak{A}$, we can remark that if $X \sim \mathrm{N}(0, \frac{\theta^2}{\sigma_1})$ and $Y \sim \mathrm{N}(0, \frac{1}{\sigma_2})$, with $X \indep Y$ we have : 
\[\proba{X+Y \geq \mu} = \frac{\sqrt{\sigma_1\sigma_2}}{2\pi\abs{\theta}}\int_\R \int_{\mu-y}^{+\infty}\exp\left(-x^2\frac{\sigma_1}{2\theta^2} -y^2\frac{\sigma_2}{2}\right) dxdy.\]

But if we set $u = x\theta$ in the preceding integral we get : 
\begin{align*}
\frac{1}{\abs{\theta}\sigma_2} \int_\R \int^{+\infty}_{\mu - u}\exp(-\frac{\sigma_1}{2\theta^2}u^2 - \frac{\sigma_2}{2}t^2) dt du = \frac{2\pi}{\sigma_2\sqrt{\sigma_1\sigma_2}}\proba{X+Y \geq \mu}.
\end{align*}

But we know that as $X$ and $Y$ are normal and independant, they form a Gaussian vector with diagonal variance matrix. So we can easily calculate the law of $X+Y = Z \sim \mathrm{N}(0, \frac{\sigma_2\theta^2 + \sigma_1}{\sigma_1\sigma_2})$, so : 
\[ 
\frac{2\pi}{\sigma_2\sqrt{\sigma_1\sigma_2}}\proba{X+Y \geq \mu} = \frac{2\pi}{\sigma_2\sqrt{\sigma_1\sigma_2}}\proba{Z \geq \mu} = \frac{\sqrt{2\pi}}{\sigma_2\sqrt{(\sigma_2\theta^2 + \sigma_1)}}\int_\mu^{+\infty} \exp\left( \frac{-z^2\sigma_1\sigma_2}{2(\sigma_2\theta^2 + \sigma_1)}\right)dz.
\]
By setting $t=z\sqrt{\frac{\sigma_1\sigma_2}{2(\sigma_2\theta^2 + \sigma_1)}}$ we have :
\begin{align*}
&\frac{1}{\sigma_2}\sqrt{\frac{2\pi}{(\sigma_2\theta^2 +\sigma_1)}}\int_\mu^{+\infty} \exp\left( \frac{-z^2\sigma_1\sigma_2}{2(\sigma_2\theta^2 + \sigma_1)}\right)dz \\
\underset{\color{white}m_\alpha \rightarrow \infty}{=}& \frac{1}{\sigma_2}\frac{2\sqrt{\pi}}{\sqrt{\sigma_1\sigma_2}}\int_{\sqrt{\frac{\sigma_1\sigma_2}{2(\sigma_2\theta^2 + \sigma_1)}}\mu}^{+\infty} \exp\left( -t^2\right)dt.\\
\intertext{Using here the lemma 1 with $N = 3$}
\underset{\mu \rightarrow \infty}{=}&\frac{1}{\sigma_2}\frac{
\sqrt{2\pi(\sigma_2\theta^2 + \sigma_1)}}{\sigma_1\sigma_2\mu}\exp\left( \frac{-\mu^2\sigma_1\sigma_2}{2(\sigma_2\theta^2 + \sigma_1)}\right)\left(1 - \frac{(\sigma_2\theta^2 + \sigma_1)}{\mu^2\sigma_1\sigma_2} + \frac{2(\sigma_2\theta^2 + \sigma_1)^2}{\mu^4\sigma_1^2\sigma_2^2} + o(\mu^{-5})\right),
\end{align*}
hence the result.
\end{proof}

\begin{lemma}\label{lma:NN:4}
For $\sigma_1 > 0$ and $\sigma_2 >0$ :
\[ \int_\R\exp(- \sigma_2x^2)\int_{t \geq \mu - x\theta}t \exp\left(- \sigma_1 t^2\right) dt dx = \frac{\sqrt{\pi}}{2\sigma_1\sqrt{\sigma_2 + \sigma_1\theta^2}} \exp\left(-\mu^2 \frac{\sigma_1\sigma_2}{\sigma_2 + \sigma_1\theta^2}\right).\]
\end{lemma}
\begin{proof}
Setting 
\[\mathfrak{L} := \int_\R\exp(- \sigma_2x^2)\int_{t \geq \mu - x\theta}t \exp\left(- \sigma_1 t^2\right) dt dx,\]
we have : 
\begin{align*}
    \mathfrak{L} &= \int_\R\exp(- \sigma_2x^2)\left[-\frac{\exp(-\sigma_1t^2)}{2\sigma_1} \right]_{\mu - x\theta}^{+\infty} dx \\
    &= \frac{1}{2\sigma_1}\int_\R\exp(- \sigma_2x^2-\sigma_1(\mu - x\theta)^2) dx \\
    &= \frac{1}{2\sigma_1}\int_\R\exp(- \sigma_2x^2-\sigma_1\mu^2 - \sigma_1x^2\theta^2 + 2\sigma_1\theta\mu x) dx \\
    &= \frac{1}{2\sigma_1} \exp(-\sigma_1\mu^2)\int_\R\exp(-x^2(\sigma_2 + \sigma_1\theta^2) + 2\sigma_1\theta\mu x) dx \\
    &= \frac{1}{2\sigma_1} \exp(-\sigma_1\mu^2)\int_\R\exp\left(-(\sigma_2 + \sigma_1\theta^2)(x^2 + 2\frac{\sigma_1\theta}{\sigma_2 + \sigma_1\theta^2}\mu x)\right) dx \\
     &= \frac{1}{2\sigma_1} \exp\left(-\mu^2(\sigma_1 - \frac{\sigma_1^2\theta^2}{\sigma_2 + \sigma_1\theta^2})\right)\int_\R\exp\left(-(\sigma_2 + \sigma_1\theta^2)(x + \frac{\sigma_1\theta}{\sigma_2 + \sigma_1\theta^2}\mu)^2\right) dx \\
     &= \frac{1}{2\sigma_1} \exp\left(-\mu^2 \frac{\sigma_1\sigma_2}{\sigma_2 + \sigma_1\theta^2}\right)\int_\R\exp\left(-(\sigma_2 + \sigma_1\theta^2)(x + \frac{\sigma_1\theta}{\sigma_2 + \sigma_1\theta^2}\mu)^2\right) dx \\
     &= \frac{\sqrt{\pi}}{2\sigma_1\sqrt{\sigma_2 + \sigma_1\theta^2}} \exp\left(-\mu^2 \frac{\sigma_1\sigma_2}{\sigma_2 + \sigma_1\theta^2}\right).
\end{align*}
\end{proof}

\subsection{Proof of Lemma~\ref{lma:esp_G_normal} }
\label{lma:NN:5}

\begin{replemma}{lma:esp_G_normal}
    We set $M = G + \xi$ where $\begin{bmatrix}
    G \\
    \xi
    \end{bmatrix} \sim \Norm{0_2}{\Sigma}$, $\Sigma = \begin{bmatrix}
        a & c\\
        c & b
    \end{bmatrix}$ with $a>0$, $b>0$ and $\abs{c}< \sqrt{ab}$. Then : 

\begin{align}
    \proba{M>m_\alpha} \underset{m_\alpha \rightarrow \infty}{=}& \frac{\sqrt{a+b+2c}}{\sqrt{2\pi}m_\alpha}\exp\left(\frac{-m_\alpha^2}{2(a+b+2c)}\right)\left[1 - \frac{a+b+2c}{m_\alpha^2} + \frac{3(a+b+2c)^2}{m_\alpha^4} + o(m_\alpha^{-5}) \right], \\
    \esp{G}[M>m_\alpha]&\underset{m_\alpha \rightarrow \infty}{=} \frac{a+c}{(a+b+2c)}m_\alpha(1 + \frac{a+b+2c}{m_\alpha^2} - 2 \frac{(a+b+2c)^2}{m_\alpha^4}+o(m_\alpha^{-5})),\\
    \esp{\xi}[M>m_\alpha]&\underset{m_\alpha \rightarrow \infty}{=}\frac{(b+c)}{(a+b+2c)}m_\alpha(1 + \frac{a+b+2c}{m_\alpha^2} - 2 \frac{(a+b+2c)^2}{m_\alpha^4}+o(m_\alpha^{-5})),\\
    \esp{G\xi}[M>m_\alpha]&\underset{m_\alpha \rightarrow \infty}{=}\frac{(a+c)(b+c)}{(a+b+2c)^2}m_\alpha^2 + \frac{(a+c)(b+c)}{(a+b+2c)} + c - 2\frac{(a+c)(b+c)}{m_\alpha^2} + o(m_\alpha^{-3}),\\
    \esp{G^2}[M>m_\alpha]&\underset{m_\alpha \rightarrow \infty}{=}\frac{(a+c)^2}{(a+b+2c)^2}m_\alpha^2 + \frac{(a+c)^2}{(a+b+2c)} + a - 2\frac{(a+c)^2}{m_\alpha^2} + o(m_\alpha^{-3}).
\end{align}
\end{replemma}
\begin{proof}

\textbf{For (1):}
As $\begin{bmatrix}
    G \\
    \xi
    \end{bmatrix} \sim \Norm{0_2}{\Sigma}$, we have $M \sim \mathrm{N}(0, a+b+2c)$ :
\begin{align*}
    \proba{M\geq m_\alpha} \underset{\color{white}m_\alpha\rightarrow\infty}{=}& \frac{1}{\sqrt{2\pi(a+b+2c)}}\int_{m_\alpha}^{+\infty} \exp(\frac{-x^2}{2(a+b+2c)})dx.\\
    \intertext{With $u = \frac{x}{\sqrt{2(a+b+2c}}$:}
    \underset{\color{white}m_\alpha\rightarrow\infty}{=}& \frac{1}{\sqrt{\pi}}\int_{\frac{m_\alpha}{\sqrt{2(a+b+2c)}}}^{+\infty} \exp(-u^2)du.\\
    \intertext{Using here the lemma 1 for $N = 3$}
    \underset{m_\alpha\rightarrow\infty}{=}& \frac{\sqrt{a+b+2c}}{\sqrt{2\pi}m_\alpha}\exp\left(\frac{-m_\alpha^2}{2(a+b+2c)}\right)\left[1 - \frac{a+b+2c}{m_\alpha^2} + \frac{3(a+b+2c)^2}{m_\alpha^4} + o(m_\alpha^{-5}) \right].
\end{align*}

\textbf{For (2):}
\begin{align*}
    \esp{G}[G + \xi \geq m_\alpha] &= \frac{1}{\alpha2\pi\sqrt{ab-c^2}} \int_\R\int_{g+x \geq m_\alpha}g\exp\left(-\frac{\delta}{2}(bg^2 - 2cgx + ax^2)\right) dg dx \\
    &= \frac{1}{\alpha2\pi\sqrt{ab-c^2}} \int_\R\exp\left(-\frac{x^2}{2b}\right)\int_{g+x \geq m_\alpha}g\exp\left(-\frac{\delta b}{2}(g - \frac{c}{b}x)^2\right)  dg dx, \\
    \intertext{with $t = g - \frac{c}{b}x$:}
    &= \frac{1}{\alpha2\pi\sqrt{ab-c^2}} \int_\R \exp\left(-\frac{x^2}{2b}\right)\int_{t \geq m - x(1 + \frac{c}{b})}(t + \frac{c}{b}x)\exp\left(-\frac{\delta b}{2}t^2\right)  dt dx.\\
\end{align*}
Splitting it in two :
\begin{align*}
\mathrm{A}_1 &:= \frac{1}{\alpha2\pi\sqrt{ab-c^2}} \int_\R \int_{t \geq m - x(1 + \frac{c}{b})}t \exp\left(-\frac{\delta b}{2}t^2 -\frac{x^2}{2b}\right)  dt dx\\
\mathrm{A}_2&:=\frac{1}{\alpha2\pi\sqrt{ab-c^2}} \frac{c}{b}\int_\R \int_{t \geq m - x(1 + \frac{c}{b})} x\exp\left(-\frac{\delta b}{2}t^2 -\frac{x^2}{2b}\right)  dt dx. \\
\end{align*}

We can use here \cref{lma:NN:4} two times, which gives :
\[
\alpha\mathrm{A}_1 = \frac{\sqrt{2\pi}(ab-c^2)^{3/2}}{b\sqrt{a+b+2c}}.
\]

And with normalisation : 
\begin{align*}
\mathrm{A}_1 \underset{m_\alpha \rightarrow +\infty}{=} \frac{(ab-c^2)m_\alpha}{b(a+b+2c)}(1 + \frac{a+b+2c}{m_\alpha^2} - 2 \frac{(a+b+2c)^2}{m_\alpha^4}+o(m_\alpha^{-5})).
\end{align*}

In the same way :
\[\alpha\mathrm{A}_2 = \frac{(b+c)c}{b\sqrt{2\pi(b+2c + a)}}  \exp\left(-m_\alpha^2 \frac{1}{2(a+b+2c)}\right). \]
And after normalisation: 
\begin{align*}
\mathrm{A}_2 \underset{m_\alpha \rightarrow \infty}{=} \frac{(b+c)c}{b(a+b+2c)}m_\alpha(1 + \frac{a+b+2c}{m_\alpha^2} - 2 \frac{(a+b+2c)^2}{m_\alpha^4}+o(m_\alpha^{-5})), \\
\end{align*}

hence the result.

\textbf{For (3):}
\begin{align*}
    \esp{\xi}[G+\xi \geq m_\alpha] &=\frac{1}{\alpha2\pi\sqrt{ab-c^2}}  \int_\R \int_{g+x \geq m_\alpha}x\exp\left(-\frac{\delta}{2}(bg^2 - 2cgx + ax^2)\right) dg dx \\
    &=\frac{1}{\alpha2\pi\sqrt{ab-c^2}}  \int_\R\exp\left(-\frac{x^2}{2b}\right)\int_{g+x \geq m_\alpha}x\exp\left(-\frac{\delta b}{2}(g - \frac{c}{b}x)^2\right).\\
    \intertext{With $t = g - \frac{c}{b}$ :}
    &=\frac{1}{\alpha2\pi\sqrt{ab-c^2}}  \int_\R\exp\left(-\frac{x^2}{2b}\right)\int_{t \geq m - x(1 + \frac{c}{b})}x\exp\left(-\frac{\delta b}{2}t^2\right)  dt dx \\
    &= \frac{b}{c} \mathrm{A}_2.
\end{align*}
So : 
\[
\esp{\xi}[G+\xi \geq m_\alpha] \underset{m_\alpha \rightarrow \infty}{=} \frac{(b+c)}{(a+b+2c)}m_\alpha(1 + \frac{a+b+2c}{m_\alpha^2} - 2 \frac{(a+b+2c)^2}{m_\alpha^4}+o(m_\alpha^{-5})).
\]

\textbf{For (4):}
\begin{align*}
    \esp[\alpha]{G\xi} &= \frac{1}{\alpha2\pi\sqrt{ab-c^2}}\int_\R\int_{g+x \geq m_\alpha}gx\exp\left(-\frac{\delta}{2}(bg^2 - 2cgx + ax^2)\right) dg dx \\
    &= \frac{1}{\alpha2\pi\sqrt{ab-c^2}}\int_\R x\exp\left(-\frac{x^2}{2b}\right)\int_{g+x \geq m_\alpha}g\exp\left(-\frac{\delta b}{2}(g - \frac{c}{b}x)^2\right)  dg dx .\\
\end{align*}
With $t =g - \frac{c}{b}x$ :
\begin{align*}
    \color{white}\esp[\alpha]{G\xi} \color{black} &= \frac{1}{\alpha2\pi\sqrt{ab-c^2}}\int_\R x\exp\left(-\frac{x^2}{2b}\right)\int_{t \geq m - x(1 + \frac{c}{b})}(t + \frac{c}{b}x)\exp\left(-\frac{\delta b}{2}t^2\right)  dt dx .\\
\end{align*}

We can divide it in two :
\begin{align*}
    \int_\R x\exp\left(-\frac{x^2}{2b}\right)\int_{t \geq m - x(1 + \frac{c}{b})}t\exp\left(-\frac{\delta b}{2}t^2\right)  dt dx, \\
    \frac{c}{b}\int_\R x^2\exp\left(-\frac{x^2}{2b}\right)\int_{t \geq m - x(1 + \frac{c}{b})}\exp\left(-\frac{\delta b}{2}t^2\right)  dt dx.\\
\end{align*}

With application of \cref{lma:NN:3} for the first integral and \cref{lma:NN:2} for the second we get : 
\begin{align*}
    \esp{G\xi}[M>m\alpha]\underset{m_\alpha \rightarrow \infty}{=}& \frac{(b+c)(ab-c^2)}{b(a+b+2c)^2}\left(m_\alpha^2 + (a+b+2c) - \frac{2(a+b+2c)^2}{m_\alpha^2} +o(m_\alpha^{-5}) \right) + \\
    &\frac{c(b+c)^2m_\alpha^2}{b(a+b+2c)^2} +\frac{c(b+c)^2}{b(a+b+2c)} - 2\frac{c(b+c)^2}{bm_\alpha^2} + c +o(m_\alpha^{-3}) \\
    \underset{m_\alpha \rightarrow \infty}{=}& \frac{(a+c)(b+c)}{(a+b+2c)^2}m_\alpha^2 + \frac{(a+c)(b+c)}{(a+b+2c)} + c - 2\frac{(a+c)(b+c)}{m_\alpha^2} + o(m_\alpha^{-3}).
\end{align*}

\textbf{For (5) :}

\begin{align*}
    \esp[\alpha]{G^2} &= \frac{1}{\alpha2\pi\sqrt{ab-c^2}}\int_\R\int_{g+x \geq m_\alpha}g^2\exp\left(-\frac{\delta}{2}(bg^2 - 2cgx + ax^2)\right) dg dx \\
    &= \frac{1}{\alpha2\pi\sqrt{ab-c^2}}\int_\R\int_{g+x \geq m_\alpha}g^2\exp\left(-\frac{\delta}{2}\left[b(g^2 - 2\frac{c}{b}gx) + ax^2 \right]\right) dg dx \\
    &= \frac{1}{\alpha2\pi\sqrt{ab-c^2}}\int_\R\int_{g+x \geq m_\alpha}g^2\exp\left(-\frac{\delta}{2}\left[b(g^2 - 2\frac{c}{b}gx + \frac{c^2}{b^2}x^2) - \frac{c^2}{b}x^2 + ax^2]\right]\right) dg dx \\
    &= \frac{1}{\alpha2\pi\sqrt{ab-c^2}}\int_\R\int_{g+x \geq m_\alpha}g^2\exp\left(-\frac{\delta b}{2}(g - \frac{c}{b}x)^2 \right) \exp\left(-\frac{\delta}{2}\left[ax^2 - \frac{c^2}{b}x^2\right]\right) dg dx \\
    &= \frac{1}{\alpha2\pi\sqrt{ab-c^2}}\int_\R\exp\left(-\frac{x^2}{2b}\right)\int_{g+x \geq m_\alpha}g^2\exp\left(-\frac{\delta b}{2}(g - \frac{c}{b}x)^2\right)  dg dx .\\
\end{align*}
We make the following changes of variables in the second integral :
\[
    t = g - \frac{c}{b}x,
\]
which give the bound : $t \geq m - x(1 + \frac{c}{b})\\$.
\begin{align*}
    \color{white}\esp[\alpha]{G^2}\color{black} &= \frac{1}{\alpha2\pi\sqrt{ab-c^2}}\int_\R\exp\left(-\frac{x^2}{2b}\right)\int_{t \geq m - x(1 + \frac{c}{b})}(t + \frac{c}{b}x)^2\exp\left(-\frac{\delta b}{2}t^2\right)  dt dx .\\
\end{align*}
We will treat the precedent integral by decomposing into 3 pieces : 
\begin{align*}
    \mathrm{I}_1 &= \int_\R\exp\left(-\frac{x^2}{2b}\right)\int_{t \geq m - x(1 + \frac{c}{b})}2\frac{c}{b}tx\exp\left(-\frac{\delta b}{2}t^2\right)  dt dx, \\
    \mathrm{I}_2 &= \int_\R\exp\left(-\frac{x^2}{2b}\right)\int_{t \geq m - x(1 + \frac{c}{b})}x^2\frac{c^2}{b^2}\exp\left(-\frac{\delta b}{2}t^2\right)  dt dx, \\
    \mathrm{I}_3 &= \int_\R\exp\left(-\frac{x^2}{2b}\right)\int_{t \geq m - x(1 + \frac{c}{b})}t^2\exp\left(-\frac{\delta b}{2}t^2\right)  dt dx. \\
\end{align*}

Using \cref{lma:NN:2} and \cref{lma:NN:3} we obtain : 
\[\esp{G^2}[M>m_\alpha]\underset{m_\alpha \rightarrow \infty}{=}\frac{(a+c)^2}{(a+b+2c)^2}m_\alpha^2 + \frac{(a+c)^2}{(a+b+2c)} + a - 2\frac{(a+c)^2}{m_\alpha^2}.\]

\end{proof}
\begin{lemma}\label{lma:NN:6}
\begin{align}
    \var{G}[M>m_\alpha] &\underset{m_\alpha \rightarrow \infty}{=} \frac{ab-c^2}{(a+b+2c)} + \frac{(a+c)^2}{m_\alpha^2}+o(m_\alpha^{-3})),\\
    \var{\xi}[M>m_\alpha] &\underset{m_\alpha \rightarrow \infty}{=} \frac{ab-c^2}{(a+b+2c)} + \frac{(b+c)^2}{m_\alpha^2}+o(m_\alpha^{-3})),\\
    \cov{G}{M}[M>m_\alpha] &\underset{m_\alpha \rightarrow \infty}{=} \frac{(a+c)(a+b+2c)}{m_\alpha^2}+o(m_\alpha^{-3}).
\end{align}
\end{lemma}
\begin{proof}
\textbf{for 7 and 8} we use that for any random variable $X$ and $Y$
\begin{align*}
    \var{X}[Y] = \esp{X^2}[Y] - \esp{X}[Y]^2.
\end{align*}

The results then follow thanks to \cref{lma:NN:5} results.

\textbf{For (9),} it's only a combination of what we have done precedently
    \begin{align*}
    \cov[\alpha]{G}{M} \underset{\color{white}m_\alpha \rightarrow \infty}{=}&\esp[\alpha]{G^2} + \esp[\alpha]{G\xi} - \esp[\alpha]{G}^2 -\esp[\alpha]{G}\esp[\alpha]{\xi}\\  
     \underset{m_\alpha \rightarrow \infty}{=}&\frac{(a+c)^2}{(a+b+2c)^2}m_\alpha^2 + \frac{(a+c)^2}{(a+b+2c)} + a - 2\frac{(a+c)^2}{m_\alpha^2} \\
     &+\frac{(a+c)(b+c)}{(a+b+2c)^2}m_\alpha^2 + \frac{(a+c)(b+c)}{(a+b+2c)} + c - 2\frac{(a+c)(b+c)}{m_\alpha^2} + o(m_\alpha^{-3}) \\
     &- (\frac{a+c}{(a+b+2c)}m_\alpha(1 + \frac{a+b+2c}{m_\alpha^2} - 2 \frac{(a+b+2c)^2}{m_\alpha^4}+o(m_\alpha^{-5})))^2 \\
     &- \frac{(b+c)}{(a+b+2c)}m_\alpha(1 + \frac{a+b+2c}{m_\alpha^2} - 2 \frac{(a+b+2c)^2}{m_\alpha^4}+o(m_\alpha^{-5}))\\
     &\times\frac{a+c}{(a+b+2c)}m_\alpha(1 + \frac{a+b+2c}{m_\alpha^2} - 2 \frac{(a+b+2c)^2}{m_\alpha^4}+o(m_\alpha^{-5}))\\
     \underset{m_\alpha \rightarrow \infty}{=}&\frac{(a+c)^2}{(a+b+2c)^2}m_\alpha^2 + \frac{(a+c)^2}{(a+b+2c)} + a - 2\frac{(a+c)^2}{m_\alpha^2} \\
     &+\frac{(a+c)(b+c)}{(a+b+2c)^2}m_\alpha^2 + \frac{(a+c)(b+c)}{(a+b+2c)} + c - 2\frac{(a+c)(b+c)}{m_\alpha^2} + o(m_\alpha^{-3}) \\
     &- \frac{(a+c)^2}{(a+b+2c)^2}m_\alpha^2(1 + \frac{a+b+2c}{m_\alpha^2} - 2 \frac{(a+b+2c)^2}{m_\alpha^4}+o(m_\alpha^{-5}))^2 \\
     &- \frac{(b+c)(a+c)}{(a+b+2c)^2}m_\alpha^2(1 + \frac{a+b+2c}{m_\alpha^2} - 2 \frac{(a+b+2c)^2}{m_\alpha^4}+o(m_\alpha^{-5}))^2\\
     \underset{m_\alpha \rightarrow \infty}{=}&\frac{(a+c)^2}{(a+b+2c)^2}m_\alpha^2 + \frac{(a+c)^2}{(a+b+2c)} + a - 2\frac{(a+c)^2}{m_\alpha^2} \\ \intertext{}
     &+\frac{(a+c)(b+c)}{(a+b+2c)^2}m_\alpha^2 + \frac{(a+c)(b+c)}{(a+b+2c)} + c - 2\frac{(a+c)(b+c)}{m_\alpha^2} + o(m_\alpha^{-3}) \\
     &- \frac{(a+c)^2}{(a+b+2c)^2}m_\alpha^2(1 + 2\frac{a+b+2c}{m_\alpha^2} - 3 \frac{(a+b+2c)^2}{m_\alpha^4}+o(m_\alpha^{-5})) \\
     &- \frac{(b+c)(a+c)}{(a+b+2c)^2}m_\alpha^2(1 + 2\frac{a+b+2c}{m_\alpha^2} - 3 \frac{(a+b+2c)^2}{m_\alpha^4}+o(m_\alpha^{-5}))\\
     \intertext{}
     \underset{m_\alpha \rightarrow \infty}{=}&a+c-\frac{(a+c)^2}{(a+b+2c)} - \frac{(b+c)(a+c)}{(a+b+2c)} + \frac{(b+c)(a+c)}{m_\alpha^2} + \frac{(a+c)^2}{m_\alpha^2}+o(m_\alpha^{-3})\\
     \underset{m_\alpha \rightarrow \infty}{=}& \frac{(b+c)(a+c)}{m_\alpha^2} + \frac{(a+c)^2}{m_\alpha^2}+o(m_\alpha^{-3})\\
     \underset{m_\alpha \rightarrow \infty}{=}& \frac{(a+c)(a+b+2c)}{m_\alpha^2}+o(m_\alpha^{-3}).\\
\end{align*}

\end{proof}

\subsection{Proof of Theorem \ref{th:corr_normal}}
\begin{reptheorem}{th:corr_normal}
With $(G, \xi) \sim \mathcal{N}(0,\Sigma)$, $M = G + \xi$, the correlation between the proxy metric $M$ and the goal $G$ goes to zero in the limit no matter the correlation between the discrepancy $\xi$ and the goal $G$.
\[
\corr{G}{M | M>m} \underset{m \rightarrow \infty}{\sim} \frac{(a+c)\sqrt{a+b+2c}}{m\sqrt{ab-c^2}}.
\]

\end{reptheorem}

\begin{proof}\label{proof:corr_normal}
    We have : 
    \[\rho_\alpha := \frac{\cov[\alpha]{M}{G}}{\sqrt{\var[\alpha]{M}\var[\alpha]{G}}},\]
    and : 
    \begin{align*}
    \cov[\alpha]{M}{G} \underset{\phantom{m_\alpha \rightarrow \infty}}{=}& \esp[\alpha]{G^2} + \esp[\alpha]{G\xi} - \esp[\alpha]{G}^2 -\esp[\alpha]{G}\esp[\alpha]{\xi}.\\ 
    \intertext{With the \cref{lma:NN:6}:}
    \underset{m_\alpha \rightarrow \infty}{=}&\frac{(a+c)^2}{(a+b+2c)^2}m_\alpha^2 + \frac{(a+c)^2}{(a+b+2c)} + a - 2\frac{(a+c)^2}{m_\alpha^2} \\
     &+\frac{(a+c)(b+c)}{(a+b+2c)^2}m_\alpha^2 + \frac{(a+c)(b+c)}{(a+b+2c)} + c - 2\frac{(a+c)(b+c)}{m_\alpha^2} + o(m_\alpha^{-3}) \\
     &- (\frac{a+c}{(a+b+2c)}m_\alpha(1 + \frac{a+b+2c}{m_\alpha^2} - 2 \frac{(a+b+2c)^2}{m_\alpha^4}+o(m_\alpha^{-5})))^2 \\
     &- \frac{(b+c)}{(a+b+2c)}m_\alpha(1 + \frac{a+b+2c}{m_\alpha^2} - 2 \frac{(a+b+2c)^2}{m_\alpha^4}+o(m_\alpha^{-5}))\\
     &\times\frac{a+c}{(a+b+2c)}m_\alpha(1 + \frac{a+b+2c}{m_\alpha^2} - 2 \frac{(a+b+2c)^2}{m_\alpha^4}+o(m_\alpha^{-5})),\\
\intertext{after simplification :}
    \underset{m_\alpha \rightarrow \infty}{=}&\frac{(a+c)(a+b+2c)}{m_\alpha^2}+o(m_\alpha^{-3}).
    \end{align*}
Then we have for the denominator : 
\begin{align*}
    \sqrt{\var[\alpha]{G}\var[\alpha]{M}} &= \sqrt{\var[\alpha]{G}\var[\alpha]{G + \xi}} \\
    &= \sqrt{\var[\alpha]{G}(\var[\alpha]{G} + \var[\alpha]{\xi} + 2\cov[\alpha]{G}{\xi})}. \\
\end{align*}

For the covariance of $\xi$ and $G$ we can use \cref{lma:NN:5},
\begin{align*}
    \cov[\alpha]{G}{\xi} \underset{\phantom{m_\alpha \rightarrow +\infty}}{=}&  \esp[\alpha]{G\xi} -\esp[\alpha]{G}\esp[\alpha]{\xi}\\\\
    \underset{m_\alpha \rightarrow +\infty}{=}& \frac{(a+c)(b+c)}{(a+b+2c)^2}m_\alpha^2 + \frac{(a+c)(b+c)}{(a+b+2c)} + c - 2\frac{(a+c)(b+c)}{m_\alpha^2} + o(m_\alpha^{-3}) \\
    &- \frac{(b+c)(a+c)}{(a+b+2c)^2}m_\alpha^2(1 + 2\frac{a+b+2c}{m_\alpha^2} - 3 \frac{(a+b+2c)^2}{m_\alpha^4}+o(m_\alpha^{-5}))^2\\
    \underset{m_\alpha \rightarrow +\infty}{=}& c - \frac{(a+c)(b+c)}{(a+b+2c)} + \frac{(a+c)(b+c)}{m_\alpha^2} + o(m_\alpha^{-3})\\
    \underset{m_\alpha \rightarrow +\infty}{=}& c - \frac{(a+c)(b+c)}{(a+b+2c)} + \frac{(a+c)(b+c)}{m_\alpha^2} + o(m_\alpha^{-3})\\
    \underset{m_\alpha \rightarrow +\infty}{=}& \frac{c^2 - ab}{(a+b+2c)} + \frac{(a+c)(b+c)}{m_\alpha^2} + o(m_\alpha^{-3}).
\end{align*}

Using \cref{lma:NN:6} for the variance, we have :
\begin{align*}
    \var[\alpha]{G} + \var[\alpha]{\xi} + 2 \cov[\alpha]{G}{\xi}\underset{m_\alpha \rightarrow \infty}{=}& \frac{ab-c^2}{(a+b+2c)} + \frac{(a+c)^2}{m_\alpha^2} + \frac{ab-c^2}{(a+b+2c)} + \frac{(b+c)^2}{m_\alpha^2} \\
    &+2\frac{c^2 - ab}{(a+b+2c)} + 2\frac{(a+c)(b+c)}{m_\alpha^2} + o(m_\alpha^{-3})\\
    \underset{m_\alpha \rightarrow \infty}{=}& \frac{(a+c)^2}{m_\alpha^2} + \frac{(b+c)^2}{m_\alpha^2} + 2\frac{(a+c)(b+c)}{m_\alpha^2} + o(m_\alpha^{-3}).\\
\end{align*}
Then : 

\begin{align*}
    \var[\alpha]{G}\var[\alpha]{M} \underset{m_\alpha \rightarrow \infty}{=}&\left(\frac{ab-c^2}{(a+b+2c)} + \frac{(a+c)^2}{m_\alpha^2}+o(m_\alpha^{-3}))\right)\times\\
    &\left(\frac{(a+c)^2}{m_\alpha^2} + \frac{(b+c)^2}{m_\alpha^2} + 2\frac{(a+c)(b+c)}{m_\alpha^2} + o(m_\alpha^{-3})\right) \\
    \underset{m_\alpha \rightarrow \infty}{=}&\frac{ab-c^2}{a+b+2c}\left(\frac{(a+c)^2}{m_\alpha^2} + \frac{(b+c)^2}{m_\alpha^2} + 2\frac{(a+c)(b+c)}{m_\alpha^2} + o(m_\alpha^{-3})\right)\\
    \underset{m_\alpha \rightarrow \infty}{=}&\frac{ab-c^2}{a+b+2c}\frac{(a+b+2c)^2}{m_\alpha^2}\\
    \underset{m_\alpha \rightarrow \infty}{=}&\frac{(ab-c^2)(a+b+2c)}{m_\alpha^2} +o(m_\alpha^{-3}).\\
\end{align*}
Hence :
\begin{align*}
    \sqrt{\var[\alpha]{G}\var[\alpha]{M}} \underset{m_\alpha \rightarrow \infty}{\sim}&\frac{\sqrt{(ab-c^2)(a+b+2c)}}{m_\alpha}.\\
\end{align*}

This finally gives : 
\begin{align*}
\rho_\alpha &\underset{m_\alpha \rightarrow +\infty}{\sim}\frac{m_\alpha}{\sqrt{(ab-c^2)(a+b+2c)}}\frac{(a+c)(a+b+2c)}{m_\alpha^2}\\
&\underset{m_\alpha \rightarrow +\infty}{\sim}\frac{(a+c)\sqrt{a+b+2c}}{m_\alpha\sqrt{ab-c^2}}.
\end{align*}

\end{proof}
\subsection{Exponential goal and heavy tail discrepancy :}

For this case, we set the goal to have an exponential law. The conditional law of the discrepancy knowing the goal is of the form (for $b \in ]1,\infty[, \eta \in ]0, \infty[$): 
\[
p_{\xi'|G}(u) =  G\exp(-G((\frac{u}{\eta})^{b-1}-1))\frac{u^{b-2}}{\eta^{b-1}}(b-1)\1\{u > \eta\}.
\]

The discrepancy defined like this follow a power law of shape parameter $b$ and position parameter $\eta$.

\subsection{Lemmas :}

We need the two following lemma that will be useful in near all of our next demonstration : 
\begin{lemma}\label{lma:EHT:1}
    Let's consider $Q$ a polynomial and $P$ a rational polynomial over an interval $I$ with $\forall x \in I, Q(x) \neq 0$ and, $ \exists K \in \R \forall x \in I, \abs{P(x)} \leq K$. We denote $f_Q$ the operation such that $f_Q(P) = \frac{\partial\frac{P}{Q}}{\partial x}$ and for $n \in \mathbb{N}$, $f_Q^n(P)$ the same operation applied $n$ times. We have then :
\begin{align*}
    &f_Q^n(P) =  \sum_{k=0}^n(-1)^k\sum_{i_0+...+i_k = n-k}\frac{P^{(i_0)}(x)Q^{(i_1+1)}(x)\dots Q^{(i_k+1)}(x)}{Q(x)^{n+k}} \sum_{\substack{0<n_1\leq i_0+1\\
    \dots\\
    n_{k-1}<n_k< n}}\prod_{j=1}^k (n_j+j).
\end{align*}
\end{lemma}
\begin{proof}
We will proceed by induction. It's to be noted that for $n > max(deg(Q),deg(P))$, many of the terms in the sum will be null, but we still denote them as a derivative of a certain order of $Q$ or $P$
\begin{align*}
    &\mathcal{P}_n : "f_Q^n(P) =  \sum_{k=0}^n(-1)^k\sum_{i_0+...+i_k = n-k}\frac{P^{(i_0)}(x)Q^{(i_1+1)}(x)\dots Q^{(i_k+1)}(x)}{Q(x)^{n+k}} \sum_{\substack{0<n_1\leq i_0+1\\
    \dots\\
    n_{k-1}<n_k< n}}\prod_{j=1}^k (n_j+j)".
\end{align*}

If $n=1$ : 
\begin{align*}
    f(P) &= \frac{\partial\frac{P}{Q}}{\partial x} \\
    &= \frac{P'}{Q} - \frac{PQ'}{Q^2}\\
    &= (-1)^0 \frac{P^{(1)}}{Q}+ (-1)^1\frac{PQ^{(1)}}{Q^2}.
\end{align*}

We have the first step. Suppose we have $n \in \mathbb{N}$ such that $\mathcal{P}_n$ is true. Let's prove that $\mathcal{P}_{n+1}$ is also true. 
\begin{align*}
    &f^{n+1}(P)\\
     =& f(f^n(P)) \\
    =&\left(\sum_{k=0}^n(-1)^k\sum_{i_0+...+i_k = n-k}\frac{P^{(i_0)}(x)Q^{(i_1+1)}(x)\dots Q^{(i_k+1)}(x)}{Q(x)^{n+k+1}} \sum_{0<n_1\leq i_0+1,\dots,n_{k-1}<n_k\leq k+\sum_{j=0}^ki_j}\prod_{j=1}^k (n_j+j) \right)'\\
    \intertext{by hypothesis}
    =&\sum_{k=0}^n(-1)^k\sum_{i_0+...+i_k = n-k}\left(\frac{P^{(i_0)}(x)Q^{(i_1+1)}(x)\dots Q^{(i_k+1)}(x)}{Q(x)^{n+k+1}}\right)' \sum_{0<n_1\leq i_0+1,\dots,n_{k-1}<n_k\leq k+\sum_{j=0}^ki_j}\prod_{j=1}^k (n_j+j).
\end{align*}
But we have :
\begin{align*}
    \left(\frac{P^{(i_0)}(x)Q^{(i_1+1)}(x)\dots Q^{(i_k+1)}(x)}{Q(x)^{n+k+1}}\right)'=& \frac{(P^{(i_0)}(x)Q^{(i_1+1)}(x)\dots Q^{(i_k+1)}(x))'Q(x)^{(n+k+1)}}{Q(x)^{2(n+k+1)}} \\
    &-\frac{(n+k+1)Q'Q^{n+k}(P^{(i_0)}(x)Q^{(i_1+1)}(x)\dots Q^{(i_k+1)}(x))}{Q(x)^{2(n+k+1)}} \\
    =& \frac{(P^{(i_0)}(x)Q^{(i_1+1)}(x)\dots Q^{(i_k+1)}(x))'}{Q(x)^{n+k+1}} \\
    &-\frac{(n+k+1)Q'(x)(P^{(i_0)}(x)Q^{(i_1+1)}(x)\dots Q^{(i_k+1)}(x)}{Q(x)^{n+k+2}}).
\end{align*}

Moreover, focusing on $(P^{(i_0)}(x)Q^{(i_1+1)}(x)\dots Q^{(i_k+1)}(x))'$:
\begin{align*}
(P^{(i_0)}(x)Q^{(i_1+1)}(x)\dots Q^{(i_k+1)}(x))' =& \sum_{j=0}^{k} P^{(i_0)}(x)Q^{(i_1+1)}(x)\dots Q^{(i_j+2)}\dots Q^{(i_k+1)}(x).
\end{align*}
So we have : 
\begin{align*}
    \left(\frac{P^{(i_0)}(x)Q^{(i_1+1)}(x)\dots Q^{(i_k+1)}(x)}{Q(x)^{n+k+1}}\right)' =& \sum_{j=0}^{k}\frac{ P^{(i_0)}(x)Q^{(i_1+1)}(x)\dots Q^{(i_j+2)}\dots Q^{(i_k+1)}(x)}{Q(x)^{n+k+1}} \\
    &-\frac{(n+k+1)Q'(x)(P^{(i_0)}(x)Q^{(i_1+1)}(x)\dots Q^{(i_k+1)}(x)}{Q(x)^{n+k+2}}).
\end{align*}
Which with the entire sum gives :
\begin{align*}
& \sum_{i_0+...+i_k = n-k}\sum_{j=0}^{k}\left(\frac{ P^{(i_0)}(x)Q^{(i_1+1)}(x)\dots Q^{(i_j+2)}\dots Q^{(i_k+1)}(x)}{Q(x)^{n+k+1}}\right)\\ &-\frac{(n+k+1)Q'(x)(P^{(i_0)}(x)Q^{(i_1+1)}(x)\dots Q^{(i_k+1)}(x)}{Q(x)^{n+k+2}}\\
=&\sum_{i_0+...+i_k = n+1-k}\frac{ P^{(i_0)}(x)Q^{(i_1+1)}(x)\dots Q^{(i_k+1)}(x)}{Q(x)^{n+k+1}} \\
&-\sum_{i_0+...+i_k = n-k}\frac{(n+k+1)Q'(x)(P^{(i_0)}(x)Q^{(i_1+1)}(x)\dots Q^{(i_k+1)}(x)}{Q(x)^{n+k+2}}.\\
\end{align*}
Plugging the second sum into the whole expression we get : 
\begin{align*} 
&\sum_{k=0}^n(-1)^{k+1}\sum_{i_0+...+i_k = n-k}\frac{Q'(x)(P^{(i_0)}(x)Q^{(i_1+1)}(x)\dots Q^{(i_k+1)}(x))}{Q(x)^{n+k+2}} \sum_{\substack{0<n_1\leq i_0+1\\
\dots\\
n_{k-1}<n_k<n}}(n+k+1)\prod_{j=1}^k (n_j+j).\\
\intertext{Taking $k' = k+1$ :}
=&\sum_{k'=1}^{n+1}(-1)^{k'}\sum_{i_0+...+i_{k'-1} = n+1 - k'}\frac{Q'(x)(P^{(i_0)}(x)Q^{(i_1+1)}(x)\dots Q^{(i_{k'-1}+1)}(x))}{Q(x)^{n+k'+1}}\sum_{\substack{0<n_1\leq i_0+1\\
\dots\\
n_{k'-2}<n_{k'-1}<n}}(n+k')\prod_{j=1}^{k'-1} (n_j+j).\\
\end{align*}

The first sum being :

\begin{align*}
\sum_{k=0}^n(-1)^{k}\sum_{i_0+...+i_k = n+1-k}\frac{ P^{(i_0)}(x)Q^{(i_1+1)}(x)\dots Q^{(i_k+1)}(x)}{Q(x)^{n+k+1}} \sum_{\substack{0<n_1\leq i_0+1\\
\dots\\
n_{k-1}<n_k<n}}\prod_{j=1}^k (n_j+j).
\end{align*}

The second one exactly complete it to top $n+1$. Indeed, the last terms of the second sum is the terms needed to complete at the rank $n+1$ the formula. Moreover, each term of the second sum complete the first sum for the case when $n_k = n$. So we have :

\begin{align*}
\sum_{k=0}^{n+1}(-1)^{k}\sum_{i_0+...+i_k = n+1-k}\frac{ P^{(i_0)}(x)Q^{(i_1+1)}(x)\dots Q^{(i_k+1)}(x)}{Q(x)^{n+k+1}} \sum_{\substack{0<n_1\leq i_0+1\\
\dots\\
n_{k-1}<n_k<n+1}}\prod_{j=1}^k (n_j+j).
\end{align*}

\end{proof}

We need, before going for the second big lemma, to calculate the derivative to any order of $Q$ : 
\begin{lemma}\label{lma:EHT:2}
If we denote $Q(x) := -g\left(\frac{m-g}{\eta}\right)^{b-1}$ with $b>2$ and $Q^{(n)}(x)$ the derivative of $Q$ to the n-th order we have : 
\[
Q^{(n)}(x) = \prod_{i=1}^{n-1}(b-i)_+ (-1)^n\frac{(m-g)^{(b-n-1)_+}}{\eta^{b-1}}(nm-bg)^{\1[b\neq n]}b^{\1[b= n]}.
\]

\end{lemma}

\begin{proof}
    Let's proceed by induction. The induction hypothesis is that for all $n \in \intervalleEntier{2}{b}$ the following property is true : 
    \[P_n : "Q^{(n)}(x) = (-1)^n\frac{(m-x)^{(b-n-1)_+}}{\eta^{b-1}}(nm-bx)^{\1[b\neq n]}b^{\1[b= n]}\prod_{i=1}^{n-1}(b-i)_+".\]
    We have : 
    \begin{align*}
        Q'(x) &=\frac{(m-x)^{(b-2)}}{\eta^{b-1}}(bx-m).
    \end{align*}
    For $n = 2$, $n<b$ :
    \begin{align*}
    \frac{\partial Q'}{\partial x}(x) &=b\frac{(m-g)^{b-2}}{\eta^{b-1}} - (b-2)\frac{(m-g)^{b-3}}{\eta^{b-1}}\\
    &= \frac{(m-g)^{b-3}}{\eta^{b-1}}(2bm + bg -b^2g - 2m) \\
    &= \frac{(m-g)^{b-3}}{\eta^{b-1}}(2m - bg)(b-1).
    \end{align*}
    For $n = 2$, $n = b$ :
    \begin{align*}
    \frac{\partial Q'}{\partial x}(x) &=\frac{b}{\eta^{b-1}}. \\
    \end{align*}
    
    So the property holds for $n = 2$.

    Now suppose we have some $n \in \mathrm{N}$ such that the induction hypothesis holds. Then we have for $n+1$, $n+1 < b$ : 
    \begin{align*}
        Q^{(n+1)}(x) &= \partial( (-1)^n\frac{(m-x)^{b-n-1}}{\eta^{b-1}}(nm-bx)\prod_{i=1}^{n-1}(b-i))/\partial x,\\
        \intertext{by induction hypothesis}
        =& (-1)^{n+1}\frac{(m-x)^{b-n-2}}{\eta^{b-1}}(bm - bx + (b-n-1)(nm-bg))\prod_{i=1}^{n-1}(b-i)\\
        =&(-1)^{n+1}\frac{(m-x)^{b-n-2}}{\eta^{b-1}}((b-n-1)nm - (b-n)bx + bm)\prod_{i=1}^{n-1}(b-i) \\
        =&(-1)^{n+1}\frac{(m-x)^{b-n-2}}{\eta^{b-1}}([bn-n^2-n+b]m - (b-n)bx)\prod_{i=1}^{n-1}(b-i) \\
        \intertext{with $(n+1)(b-n) = bn-n^2+b-n$}
        =&(-1)^{n+1}\frac{(m-x)^{b-n-2}}{\eta^{b-1}}((n+1)m - bg)\prod_{i=1}^{n}(b-i). \\
    \end{align*}
    If $n+1 = b$ : 
    \begin{align*}
        Q^{(n+1)}(x) &= \partial( (-1)^n\frac{(nm-bx)}{\eta^{b-1}}\prod_{i=1}^{n-1}(b-i)_+)/\partial x\\
        \intertext{by induction hypothesis}
        =& (-1)^{n+1}\frac{\prod_{i=0}^{n-1}(b-i)}{\eta^{b-1}}\\
        \intertext{as $b = n+1$, $b-n = 1$ :}
        =& (-1)^{n+1}\frac{\prod_{i=0}^{n}(b-i)}{\eta^{b-1}}.\\
    \end{align*}
    The  part $n > b$ is trivial.
\end{proof}

\begin{lemma}\label{lma:EHT:3}
For $l, k \in \mathbb{N}$, for any $N\in \mathbb{N}$, if we denote $Q(g) := -g\left(\frac{m-g}{\eta}\right)^{b-1}$ and $P_l(g) := \frac{g^l}{(m-g)^s}$, $b \in (1,\infty)$ and $\eta \in (0,\infty)$ we have :
\[
    \int_0^{m-\eta}\frac{g^l}{(m-g)^s}\exp(-g\left(\frac{m-g}{\eta}\right)^{b-1})dg \underset{m \rightarrow \infty}{=} \sum_{n=0}^{N}\left[ \frac{f_{Q'}^{n}(P_l(g))}{Q'(g)}\exp(Q(g)) \right]_0^{\frac{m}{b+1}} + o(\frac{1}{m^{(N+1)b-l+s-1}}).
\]

\end{lemma}
\begin{proof}
First we have to cut it in half : 
\begin{align*}
\int_0^{m-\eta}\frac{g^l}{(m-g)^s}\exp(-g\left(\frac{m-g}{\eta}\right)^{b-1})dg = &\underbrace{\int_0^{m/(b+1)}\frac{g^l}{(m-g)^s}\exp(-g\left(\frac{m-g}{\eta}\right)^{b-1})dg}_{I_1} \\ &+\underbrace{\int_{m/(b+1)}^{m-\eta}\frac{g^l}{(m-g)^s}\exp(-g\left(\frac{m-g}{\eta}\right)^{b-1})dg}_{I_2}.
\end{align*}

$I_2$ can be roughly bounded as it will be negligible : 
\begin{align*}
    \int_{m/(b+1)}^{m-\eta}\frac{g^l}{(m-g)^s}\exp(-g\left(\frac{m-g}{\eta}\right)^{b-1})dg \leq \frac{(m-\eta)^{l+1}}{\eta^{s}}\exp(-\frac{m}{b+1}).
\end{align*}

For the first one, we note $Q(x) = -g\left(\frac{m-g}{\eta}\right)^{b-1}$. The derivative of $Q$ is different from $0$ on the whole interval $[0, \frac{m}{b+1}]$ as it is equal to $Q'(x) = \frac{(m-g)^{b-2}}{\eta^{b-1}}(bg-m)$. If we denote $P_l(g) = \frac{g^l}{(m-g)^{s}}$. With integration by part we have : 
\begin{align*}
    &\int_0^{m/(b+1)}g^l\exp(-g\left(\frac{m-g}{\eta}\right)^{b-1})dg \\
    =& \left[ \frac{P_l(g)}{Q'(g)}\exp(Q(g)) \right]_0^{\frac{m}{b+1}} - \int_0^{m/(b+1)}\frac{\partial\frac{P}{Q'}}{\partial x}(g)\exp(Q(g))dg \\
    =& \left[ \frac{P_l(g)}{Q'(g)}\exp(Q(g)) \right]_0^{\frac{m}{b+1}} - \int_0^{m/(b+1)}f_{Q'}(P_l)(g)\exp(Q(g))dg, \\
    \intertext{iterating $N+1$ times we get (with $f^{0}$ as the identity):}
    =& \underbrace{\sum_{n=0}^{N}\left[ \frac{f_{Q'}^{n}(P_l(g))}{Q'(g)}\exp(Q(g)) \right]_0^{\frac{m}{b+1}}}_{S} + (-1)^{N+1} \int_0^{m/(b+1)}f_{Q'}^{N+1}(P_l)(g)\exp(Q(g))dg. \\
\end{align*}

We need to show that $\int_0^{m/(b+1)}f_{Q'}^{N+1}(P_l)(g)\exp(Q(g))dg$ is negligible in front of $S$. Using \cref{lma:EHT:1}, we  know that for any $n\in \mathbb{N}$ :
\begin{align*}
f_{Q'}^{n}(P_l)(x)\exp(Q(x)) =&\sum_{k=0}^n(-1)^k\sum_{\substack{i_0+...+i_k = n-k \\ i_0 \leq l \\ i_1 \leq b-2\\ \vdots \\i_k \leq b-2} }\frac{P^{(i_0)}(x)Q^{(i_1+2)}(x)\dots Q^{(i_k+2)}(x)}{Q'(x)^{n+k}} \sum_{\substack{0<n_1\leq i_0+1\\
    \dots\\
    n_{k-1}<n_k< n}}\prod_{j=1}^k (n_j+j).\\
\end{align*}

Using \cref{lma:EHT:2}, we have for any $i<b$ 
\begin{align*}
    Q^{(i+2)}(x) =& \prod_{s=1}^{i-1}(b-s)_+ (-1)^{i+2}\frac{(m-x)^{b-i-3}}{\eta^{b-1}}((i+2)m-bg). 
\end{align*}

And using Leibniz rule we get for $P$: 
\begin{align*}
    P^{(i_0)} = \sum_{k = 0}^{i_0} \binom{i_0}{k}\prod_{j = 0}^{k}(l-j)\prod_{i = 0}^{n-k}(s-j)g^{l-k}(m-g)^{-s-i_0+k}.
\end{align*}

We can already bound the derivative of $P^{i_0}$ by the fact that the integration is within $[0,m/(b+1)]$ range :
\begin{align*}
    P^{(i_0)} \leq \sum_{k = 0}^{i_0} \binom{i_0}{k}\prod_{j = 0}^{k}(l-j)\prod_{i = 0}^{i_0-k}(s-j)m^{l-s-i_0}(b+1)^{s+i_0-l}.
\end{align*}

Then for a given set of $\mathcal{I} = \{i_0, i_1, ... , i_k\}$ with $\forall j \in [0,k], i_j > 0$ and $i_0+i_1+... +i_k = n-k$:
\begin{align*}
    &\frac{P^{(i_0)}(x)Q^{(i_1+2)}(x)\dots Q^{(i_k+2)}(x)}{Q'(x)^{n+k}} \\
    \leq& \sum_{q = 0}^{i_0} \binom{i_0}{q}\frac{(b+1)^{s+i_0-l}l!s!(b-1)!^k(m-x)^{kb+i_0-n-2k}\prod_{i \in \mathcal{I}}((i+2)m - bx)}{(l-q)!(s-q)!\eta^{k(b-1)}(b-i_1-2)!\dots(b-i_k-2)!} \times\\
    &\frac{m^{l-s-i_0}\eta^{n+k(b-1)}}{(m-x)^{(n+k)(b-2)}(m-bx)^{n+k}},
    \intertext{using the fact that as $x \in [0,m/b+1]$ we have $\frac{(i+2)m - bx}{m -bx} \leq (b+1)(i+1)+1 \leq (b+1)(n+1)+1$ and simplifying :}
    \leq& \sum_{q = 0}^{i_0} \binom{i_0}{q}\frac{(b+1)^{s+i_0-l}l!s!(b-1)!^k((b+1)(n+1)+1)^k}{(l-q)!(s-q)!(b-i_1-2)!\dots(b-i_k-2)!} \times \frac{m^{l-s-i_0}\eta^{n}}{(m-x)^{n(b-1)-i_0}(m-bx)^{n}}\\
    \leq&\sum_{q = 0}^{i_0} \binom{i_0}{q}\frac{(b+1)^{s+i_0-l}l!(b-1)!^k((b+1)(n+1)+1)^k}{(l-i_0)!(b-i_1-2)!\dots(b-i_k-2)!} \times \frac{\eta^{n}(b+1)^{nb - i_0}}{m^{nb-l+s}b^{n(b-1)-i_0}}.\\
\end{align*}
It's to be noted that we can extend with no difficulty the bound to the case where one or several of the $i_j$ are superior or equal to $b$ : if it's strictly superior the bound is trivial as the quantity is $0$, and if it's $b$ we simply have a constant. This allow us to extend the result to the case where $b=2$.

Using the fact that the preceding majoration is for any combination $\{i_0, i_1, ... , i_k\}$ with $\forall j \in [0,k], i_j > 0$ and $i_0+i_1+... +i_k = N+1-k$, we have for a $C$ sufficiently big that does not depend on $m$:

\begin{align*}
    m^{(N+1)b-l+s-1}\abs{\int_0^{m/(b+1)}f_{Q'}^{N+1}(P_l)(g)\exp(Q(g))dg} \underset{\phantom{m\rightarrow \infty}}{\leq}& C\frac{m^{(N+1)b-l+s-1}}{m^{(N+1)b-l+s}} \times \abs{\int_0^{m/(b+1)}\exp(Q(g))dg}\\
    \underset{\phantom{m\rightarrow \infty}}{\leq}& C\frac{1}{m} \times \abs{\int_0^{m/(b+1)}\exp(-g)dg}\\
    \underset{\phantom{m\rightarrow \infty}}{\leq}& C\frac{1}{m} \times(1 - \exp(-m/(b+1)))\\
    \underset{m\rightarrow \infty}{\rightarrow}& 0.
\end{align*}

This concludes the proof.

\end{proof}

The following lemmas are the building blocks necessary to calculate all the quantities we are interested in after, and are consequences of \cref{lma:EHT:3}. It's to be noted there that there is no possible integrability issue as all integrals are well defined and convergent for $b \in (1,\infty)$ and $\eta \in (0,\infty)$ :

\begin{lemma}\label{lma:EHT:4}
\begin{align*}
\int_0^{m-\eta}\exp(-g\left(\frac{m-g}{\eta}\right)^{b-1})dg \underset{m \rightarrow \infty}{=}& \frac{\eta^{b - 1}}{m^{b-1}}  + 2 \frac{\eta^{2 b - 2}\left(b - 1\right)}{m^{1 - 2 b}} + \frac{3 \eta^{3 b - 3}\left(b - 1\right) \left(3 b - 2\right)}{m^{3b-1}}  \\
&+ 8 \frac{\eta^{4 b - 4}\left(b - 1\right) \left(2 b - 1\right) \left(4 b - 3\right)}{m^{4b-1}} + o\left(\frac{1}{m^{4b-1}}\right).
\end{align*}
\end{lemma}

\begin{proof}
Use \cref{lma:EHT:3} with $l = 0, s = 0$ and $n = 3$.
\end{proof}

\begin{lemma}\label{lma:EHT:5}
\[\int_0^{m-\eta}g\exp(-g\left(\frac{m-g}{\eta}\right)^{b-1})dg \underset{m \rightarrow \infty}{=}\frac{\eta^{2 b - 2}}{m^{2 b-2}} + \frac{6\left(b - 1\right) \eta^{3 b - 3}}{m^{3b-2}}  + \frac{12 \eta^{4 b - 4}\left(b - 1\right) \left(4 b - 3\right)}{m^{4b-2}} + o\left(\frac{1}{m^{4b-2}}\right).
\]
\end{lemma}

\begin{proof}
Use \cref{lma:EHT:3} with $l = 1, s = 0$ and $n = 3$.
\end{proof}

\begin{lemma}\label{lma:EHT:6}
\[\int_0^{m-\eta}g^2\exp(-g\left(\frac{m-g}{\eta}\right)^{b-1})dg \underset{m\rightarrow\infty}{=} \frac{2 \eta^{3 b - 3}}{m^{3b-3}} + \frac{24 \eta^{4 b - 4}\left(b - 1\right)}{m^{4b-3}} + o\left(\frac{1}{m^{4b-3}}\right).  \]
\end{lemma}

\begin{proof}
Use \cref{lma:EHT:3} with $l = 2, s = 0$ and $n = 3$.
\end{proof}

\begin{lemma}\label{lma:EHT:7}
\[\int_0^{m-\eta}\frac{\eta^{b-1}}{(m-u)^{b-1}}\exp(-u\left(\frac{m-u}{\eta}\right)^{b-1}) du \underset{m\rightarrow\infty}{=} \frac{\eta^{2 b - 2}}{m^{2b-2}} + \frac{3 \eta^{3 b - 3}\left(b - 1\right)}{m^{3b-2}}  + \frac{4 \eta^{4 b - 4}\left(b - 1\right) \left(4 b - 3\right)}{m^{4b-2}} + o\left(\frac{1}{m^{4b-2}}\right). \]
\end{lemma}

\begin{proof}
Use \cref{lma:EHT:3} with $l = 0, s = b-1$ and $n = 3$.
\end{proof}

\begin{lemma}\label{lma:EHT:8}
\begin{align*}
\int_0^{m-\eta} \left(\frac{\eta^{b-1}}{(m-u)^{b-2}}\right)exp(-u\left(\frac{m-u}{\eta}\right)^{b-1}) du \underset{m \rightarrow \infty}{=}&\frac{\eta^{2 b - 2}}{m^{2b-3}}  - \frac{\eta^{3 b - 3}\left(4 - 3 b\right)}{m^{3b-3}} + \frac{4 \eta^{4 b - 4}\left(b - 1\right) \left(4 b - 5\right)}{m^{4b-3}} \\
&+ \frac{5 \eta^{5 b - 5}\left(b - 1\right) \left(5 b - 6\right) \left(5 b - 4\right)}{m^{5b-3}} + o\left(\frac{1}{m^{5b-3}}\right).
\end{align*}
\end{lemma}

\begin{proof}
Use \cref{lma:EHT:3} with $l = 0, s = b-2$ and $n = 3$.
\end{proof}

\begin{lemma}\label{lma:EHT:9}
\begin{align*}
\int_0^{m-\eta} \left(\frac{\eta^{2b-2}}{(m-u)^{2b-2}}\right)exp(-u\left(\frac{m-u}{\eta}\right)^{b-1}) \underset{m \rightarrow \infty}{=}& \frac{\eta^{3 b - 3}}{m^{3b-3}}  + \frac{4\eta^{4 b - 4}\left(b - 1\right)}{m^{4b-3}} + \frac{5 \eta^{5 b - 5}\left(b - 1\right) \left(5 b - 4\right)}{m^{5b-3}}\\
&+ \frac{12 \left(b - 1\right) \left(3 b - 2\right) \left(6 b - 5\right) \eta^{6 b - 6}}{m^{6b-3}} + o\left(\frac{1}{m^{6b-3}}\right).
\end{align*}
\end{lemma}

\begin{proof}
Use \cref{lma:EHT:3} with $l = 0, s = 2b-2$ and $n = 3$.
\end{proof}

\begin{lemma}\label{lma:EHT:10}
If $G\sim \mathcal{E}(1)$ and the conditionnal density of $\xi$ is $p_{\xi|G}(x) :=  G\exp(-G((\frac{x}{\eta})^{b-1}-1))\frac{x^{b-2}}{\eta^{b-1}}(b-1)\1[x > \eta]$, then we have : 
\begin{align*}
\proba{G+\xi > m} \underset{m\rightarrow\infty}{=}& \frac{\eta^{b - 1}}{m^{b-1}} + \frac{2 \eta^{2 b - 2}(b - 1)}{m^{2b-1}}  + \frac{3 \eta^{3 b - 3}(b - 1)(3 b - 2)}{ m^{3 b - 1}}   + \frac{8 \eta^{4 b - 4}(b - 1)(2 b - 1)(4 b - 3)}{m^{4b-1}} \\
&+ o\left(\frac{1}{m^{1 - 4 b}}\right).\\
\end{align*}
\end{lemma}
\begin{proof}
\begin{align*}
    \proba{G+\xi > m} &= \int_\R\exp(-g) \int_\R \exp(-g(\left(\frac{x}{\eta}\right)^{b-1} - 1)x^{b-2}\frac{g}{\eta^{b-1}}(b-1) \1[x>\eta, g >0, g+x>m]dx dg \\
    &= \int_0^\infty\exp(-g) \int_\R \exp(-g(\left(\frac{x}{\eta}\right)^{b-1} - 1)x^{b-2}\frac{g}{\eta^{b-1}}(b-1) \1[x>\eta, x>m-g]dx dg \\
    &= \int_0^\infty\exp(-g) \int_\R \exp(-g(\left(\frac{x}{\eta}\right)^{b-1} - 1)x^{b-2}\frac{g}{\eta^{b-1}}(b-1) \1[x>\max(\eta, m-g)]dx dg, \\
    \intertext{as $g >m-\eta \implies \eta >m-g$}
    &=\int_0^{m-\eta}\exp(-g) \underbrace{\int_{m-g}^{\infty} \exp(-g(\left(\frac{x}{\eta}\right)^{b-1} - 1)x^{b-2}\frac{g}{\eta^{b-1}}(b-1) dx}_{A_1} dg + S_G(m-\eta). \\
\end{align*}

The survival function of $G$ gives : 
\begin{align*}
    S_G(m-\eta) &= \int_{m-\eta}^\infty\exp(-g) \\
    &= \exp(-m) \exp(\eta)\\
    &= o(\frac{1}{m^{2b}}).
\end{align*}

$A_1$ gives : 
\begin{align*}
    A_1 = \int_{m-g}^{\infty} \exp(-g\left(\frac{x}{\eta}\right)^{b-1}x^{b-2}\frac{g}{\eta^{b-1}}(b-1) dx \underset{b>1}{=}& \left[  \exp(-g\left(\frac{x}{\eta}\right)^{b-1})\right]^\infty_{m-g} \\
    \underset{b>1}{=}&\exp(-g\left(\frac{m-g}{\eta}\right)^{b-1}). \\
\end{align*}

So we have : 
\[
    \proba{G+\xi > m} \underset{m\rightarrow \infty}{=} \int_0^{m-\eta}\exp(-g\left(\frac{m-g}{\eta}\right)^{b-1})dg + o(\frac{1}{m^{2b}}).
\]

Applying \cref{lma:EHT:4} then yield the results

\end{proof}

\subsection{Proof of lemma 3.3 :}
\begin{replemma}{lma:esp_goal_heavy}
    \begin{align*}
    \esp{G}[M>m] \underset{m\rightarrow\infty}{=}&\frac{\eta^{b - 1}}{m^{b-1}} + \frac{4 \eta^{2 b - 2}(b - 1) }{m^{1 - 2 b}} + \frac{\eta^{3 b - 3}(b - 1) (31 b - 22)}{m^{3b-1}}  - & \\ 
    &\frac{2(b - 1) (3 b - 2) (27 b - 23)\eta^{4 b - 4}}{m^{4b-1}}  + o\left(\frac{1}{m^{4b-1}}\right),&\\
    \end{align*}
\end{replemma}

\begin{proof}
Denoting by $\alpha := \proba{G+\xi >m}$ :
\begin{align*} \alpha\esp{G}[G+\xi >m] &= \int_0^{m-\eta}\exp(-g)g \int_{m-g}^{\infty} \exp(-g(\left(\frac{x}{\eta}\right)^{b-1} - 1))x^{b-2}\frac{g}{\eta^{b-1}}(b-1) dx dg + \int_{m-\eta}^{\infty}\exp(-g)g dg\\
    &= \int_0^{m-\eta}\exp(-g)g \int_{m-g}^{\infty} \exp(-g(\left(\frac{x}{\eta}\right)^{b-1} - 1)x^{b-2}\frac{g}{\eta^{b-1}}(b-1) dx dg + \int_{m-\eta}^{\infty}g\exp(-g) dg\\
    &= \underbrace{\int_0^{m-\eta}g\exp(-g\left(\frac{m-g}{\eta}\right)^{b-1})dg}_{A_2} + \underbrace{\int_{m-\eta}^{\infty}g\exp(-g) dg}_{A_3}.\\
\end{align*}
Focusing first on $A_3$ : 
\begin{align*}
    \int_{m-\eta}^{\infty}g\exp(-g) dg &= \left[-g\exp(-g)\right]_{m-\eta}^{\infty} + \int_{m-\eta}^{\infty}\exp(-g) dg\\
    &= (m-\eta)\exp(-(m-\eta)) + \exp(-(m - \eta))\\
    &= \exp(-m)\exp(\eta)(m - \eta + 1).
\end{align*}

For $A_2$, we can apply lemma \ref{lma:EHT:5}.

We still need an equivalent to normalize : 
\begin{align*}
    &\frac{1}{\alpha}\\
    \underset{m \rightarrow \infty}{=}&\frac{1}{\frac{\eta^{b - 1}}{m^{b-1}} + \frac{2 \eta^{2 b - 2}(b - 1)}{m^{2b-1}}  + \frac{3 \eta^{3 b - 3}(b - 1)(3 b - 2)}{ m^{3 b - 1}}   + \frac{8 \eta^{4 b - 4}(b - 1)(2 b - 1)(4 b - 3)}{m^{4b-1}} + o\left(\frac{1}{m^{1 - 4 b}}\right)}\\
    \underset{m \rightarrow \infty}{=}&\frac{ m^{b-1}}{\eta^{b - 1}(1 +2 \eta^{b - 1} m^{- b} \left(b - 1\right) + 3 \eta^{2 b - 2} m^{- 2 b} \left(b - 1\right) \left(3 b - 2\right) + 8 \eta^{3 b - 3} m^{- 3 b} \left(b - 1\right) \left(2 b - 1\right) \left(4 b - 3\right) +
    o\left(\frac{1}{m^{3 b}}\right))},
    \intertext{using the classic development for geometric series :}\\
    \underset{m \rightarrow \infty}{=}& \frac{ m^{b-1}}{\eta^{b - 1}}(1 - \frac{2 \eta^{b - 1}(b - 1)}{m^{b}} - \frac{\eta^{2 b - 2}(5 b - 2)(b - 1)}{m^{2 b}} - \frac{4 \eta^{3 b - 3}(b - 1)(3 b - 2)(3 b - 1)}{ m^{3b}}+ O\left(\frac{1}{m^{4b}}\right))\\
    \underset{m \rightarrow \infty}{=}& \frac{m^{b - 1}}{\eta^{b-1}}  - \frac{2 \left(b - 1\right)}{m} - \frac{\eta^{b - 1}(b - 1) (5 b - 2)}{m^{b + 1}} - \frac{4 \eta^{2 b - 2}(b - 1)(3 b - 2)(3 b - 1)}{m^{2 b + 1}}+ O\left(\frac{1}{m^{3b +1}}\right)\\
    =: \Delta    
\end{align*}

As we will use it several time in the future, we denote it by $\Delta$.
Multiplying the result of the \cref{lma:EHT:5} and this yields the results.
\end{proof}

\subsection{Proof of lemma \ref{lma:esp_noise_heavy} :}

\begin{replemma}{lma:esp_noise_heavy}
    \begin{align*}
    \esp{\xi}[M>m] \underset{m\rightarrow\infty}{=}& m\frac{b - 1}{b - 2} - \frac{2\eta^{b - 1}(b - 1)}{m^{b-1}(b - 2)} - \frac{8 \eta^{2 b - 2}(b - 1)^{2}}{m^{2b-1}(b - 2)} - \frac{2 \eta^{3 b - 3}(b - 1)^{2} (31b - 22)}{m^{3b-1}(b - 2)}   & (3)\\
    &+ o\left(\frac{1}{m^{3b-1}}\right),&\\
    \end{align*}
\end{replemma}

\begin{proof}
\begin{align*}
    \alpha\esp{\xi}[G+\xi >m] =& \int_0^\infty \int_\R \exp(-g(\left(\frac{x}{\eta}\right)^{b-1})x^{b-1}\frac{g}{\eta^{b-1}}(b-1) \1[x>\max(\eta, m-g)]dx dg \\
    \intertext{denoting $M(x) := max(0,m-x)$}
    =& \int_\eta^\infty \frac{x^{b-1}(b-1)}{\eta^{b-1}}\underbrace{\int_{M(x)}^{+\infty} g\exp(-g(\left(\frac{x}{\eta}\right)^{b-1}) dg}_{A_4} dx. \\
\end{align*}
Calculating first $A_4$ with an integration by part : 
\begin{align*}
    \int_{M(x)}^{+\infty} g\exp(-g(\left(\frac{x}{\eta}\right)^{b-1}) dg =&  M(x)\left(\frac{\eta}{x}\right)^{b-1}\exp(-M(x)\left(\frac{x}{\eta}\right)^{b-1}) +\int_{M(x)}^{+\infty}\frac{\eta^{b-1}}{x^{b-1}}\exp(-g\left(\frac{x}{\eta}\right)^{b-1}) dg.\\
\end{align*}
Taking it into the full integral : 
\begin{align*}
    &\alpha\esp{\xi}[G+\xi >m]\\
    =&\int_\eta^\infty (b-1)
    M(x)\exp(-M(x)\left(\frac{x}{\eta}\right)^{b-1}) + (b-1)\int_{M(x)}^{+\infty}\exp(-g\left(\frac{x}{\eta}\right)^{b-1}) dg dx \\
    =&\int_\eta^\infty (b-1)
    M(x)\exp(-M(x)\left(\frac{x}{\eta}\right)^{b-1}) + (b-1)\frac{\eta^{b-1}}{x^{b-1}}\exp(-M(x)\left(\frac{x}{\eta}\right)^{b-1}) dx \\
    =&\int_\eta^m (b-1)
    (m-x)\exp(-(m-x)\left(\frac{x}{\eta}\right)^{b-1}) + (b-1)\frac{\eta^{b-1}}{x^{b-1}}\exp(-(m-x)\left(\frac{x}{\eta}\right)^{b-1}) dx \\
    & + \int_m^\infty (b-1)\frac{\eta^{b-1}}{x^{b-1}} dx\\
    =& \underbrace{\int_0^{m-\eta} (b-1)
    u\exp(-u\left(\frac{m-u}{\eta}\right)^{b-1}) du}_{A_5} + \underbrace{\int_0^{m-\eta} du(b-1)\frac{\eta^{b-1}}{(m-u)^{b-1}}\exp(-u\left(\frac{m-u}{\eta}\right)^{b-1}) du}_{A_6} \\
    & + \frac{(b-1)\eta^{b-1}}{(b-2)m^{b-2}}.\\
\end{align*}

The result of \cref{lma:EHT:5} gives an equivalent for $A_5$, \cref{lma:EHT:7} gives an equivalent for $A_6$. Summing yields the expected value not normalised by the probability: 
\[
\alpha\esp{\xi}[G+\xi >m] \underset{m \rightarrow \infty}{=}\frac{\eta^{2 b - 2}}{m^{2b-2}}  + \frac{6 \eta^{3 b - 3}(b - 1)}{m^{3b-2}} + \frac{12 \eta^{4 b - 4}(b - 1)(4 b - 3)}{m^{4b-2}}   + o\left(\frac{1}{m^{4b-2}}\right).
\]
Mutliplying by $\Delta$ yields the result.
\end{proof}

\subsection{Heavy tailed goal and light tailed discrepancy :}\label{sec:proof_HTLT}

\begin{lemma}\label{lma:HTLT:conditionnal_converge_as}
Consider any random vector $\begin{pmatrix}X\\
Y \end{pmatrix}$, where $X$ is light tailed. Then we have for some $s >0$ : 
\[
\frac{\proba{X>t}[Y]}{e^{-st}} \overset{a.s}{\underset{t \rightarrow \infty}{\rightarrow}} 0.
\]
\end{lemma}

\begin{proof}

First, by definition of the light tailedness of $X$, we now that there exist $s > 0 $ such that : 
\[
\frac{\proba{X>t}}{e^{-st}} \underset{t \rightarrow \infty}{\rightarrow} 0
\]

We have by the tower property of the conditional expectation: 
\begin{align*}
    \esp{e^{st}\proba{X>t}[Y]} \underset{\phantom{t\rightarrow \infty}}{=} & e^{st}\proba{X>t}\\
    \underset{t\rightarrow \infty}{\rightarrow}& 0
\end{align*}

First, let's consider the discrete sequence of random variables : $\forall n \in \N, Z_n = \proba{Z_n > n}[Y]$, with $\esp{Z_n} = \proba{X>n}$

Let $0< c < s$, $\varepsilon > 0$. By markov inequality we have : 
\begin{align*}
    \proba{e^{cn}Z_n \geq \varepsilon} \leq& \frac{\esp{Z_n}e^{cn}}{\varepsilon} \\
    =&\frac{\proba{X > n}e^{cn}}{\varepsilon}\\
\end{align*}

So considering the sum we have: 
\begin{align*}
    \sum_{n\geq 1}\proba{e^{cn}Z_n \geq \varepsilon} \leq& \sum_{n\geq 1}\frac{\esp{Z_n}e^{cn}}{\varepsilon} \\
    =& \sum_{n\geq 1}\frac{\proba{X > n}e^{cn}}{\varepsilon}\\
\end{align*}

However we know that $\proba{X>n} = o(e^{sn})$, for $s > c$. As such we know that : 
\[
\sum_{n\geq 1}\frac{\proba{X > n}e^{cn}}{\varepsilon} < \infty
\]
This is true for any $\varepsilon >0$, which mean that we have by the Borel-Cantelli lemma :
\[
e^{cn}Z_n \underset{n \rightarrow \infty}{\overset{a.s}{\rightarrow}} 0 
\]
We now consider the continuous version of $Z_n$ : $\forall t \in \R^+, Z_t = \proba{X>t}[Y]$.

Let $\omega \in \Omega$. $Z_t(\omega)$ is monotonously decreasing as it's a probability, as such we have : 
\[
\forall t \in \R, Z_t(\omega) \leq Z_{\lfloor t \rfloor}(\omega)
\]

Moreover : 
\begin{align*}
    \forall t \in \R^+,\exists x \in [0,1)/ \exp(ct) = \exp(c(\lfloor t \rfloor +x))\\
    \implies \forall t\in \R^+, \exp(ct) \leq \exp(c\lfloor t \rfloor)\exp(c)
\end{align*}

Combining bot we get 
\begin{align*}
    Z_t(\omega)e^{ct} \underset{\phantom{t \rightarrow \infty}}{\leq}& e^ce^{c\lfloor t \rfloor}Z_{\lfloor t \rfloor}(\omega)\\
    \underset{t \rightarrow \infty}{\rightarrow}& 0 
\end{align*}
This happen on some event $\Omega_0$ with $\proba{\Omega_0} = 1$ as we have $e^{cn}Z_n \underset{n \rightarrow \infty}{\overset{a.s}{\rightarrow}} 0 $

We can conclude : 
\[
    Z_te^{ct} \underset{t \rightarrow \infty}{\overset{a.s}{\rightarrow}} 0
\]
\end{proof}

\begin{reptheorem}{th:HTLT:No_goodhart}

Suppose that G is heavy tailed, ie denoting $S_G$ the survival function of $G$ we have : $S(G) = \frac{l(t)}{t^{1/\gamma}}$ where $l$ is a slowly varying function (ie $\forall a >0, \underset{t \rightarrow \infty}{\lim}\frac{l(at)}{l(t)} = 1$) and $\xi$ is a right and left light tailed random variable, ie $\exists s \in \R / \frac{\proba{\xi>t}}{e^{-st}} \underset{t \rightarrow \infty}{\rightarrow} 0 \text{ and } \frac{\proba{\xi<t}}{e^{-s|t|}} \underset{t \rightarrow -\infty}{\rightarrow} 0$, and define $M := G+\xi$. Then : 

\begin{align*}
\esp{G}[M>m] \underset{m \rightarrow \infty}{\geq}& (m+t_m)\frac{(l(m+t_m) - l(m)(\frac{1}{m} - t_m/m^{2})^{1/\gamma})}{l(m+t_m)}\\
\underset{m \rightarrow \infty}{\sim}& m+ o(m)
\end{align*}

\end{reptheorem}

\begin{proof}

We have : 
\[
    \int_\R g p_G(g) \int_{x \geq m-g}p_{\xi|G}(x) dx dg = \underbrace{\int_{\R^-} g p_G(g) \int_{x \geq m-g}p_{\xi|G}(x) dx dg}_{I_1} + \underbrace{\int_{\R^+} g p_G(g) \int_{x \geq m-g}p_{\xi|G}(x) dx dg}_{I_2}.\\
\]

Finding an equivalent for $I_1$ first :

\begin{align*}
    |\int_{\R^-} g p_G(g) \int_{x \geq m-g}p_{\xi|G}(x) dx dg| <& \int_{\R^-} |g| p_G(g) \int_{x \geq m-g}p_{\xi|G}(x) dx dg \\
    \intertext{using the fact that $g$ is negative :}
    <&\int_{\R^-} |g| p_G(g) \int_{x \geq m}p_{\xi|G}(x) dx dg\\
    =&\esp{|G|\1[G \in \R^{-}]\proba{\xi > m}[G]}.
\end{align*}
Using lemma \ref{lma:HTLT:conditionnal_converge_as} and the fact that $\xi$ is light tail, we know that there exist $c \in \R$ such that $e^{ct}\proba{\xi > m}[G] \underset{m \rightarrow \infty}{\rightarrow} 0$. Using Slutsky's theorem we have : 
\[
e^{cm}\proba{\xi > m}[G]|G|\1[G \in \R^{-}] \overset{\mathcal{L}}{\underset{m \rightarrow \infty}{\rightarrow}} 0.
\]
Hence :
\[
e^{cm}
\esp{\proba{\xi > m}[G]|G|\1[G \in \R^{-}]} \underset{m \rightarrow \infty}{\rightarrow} 0,
\]

as such we have : 
\[
I_1 =  \int_{\R^-} g p_G(g) \int_{x \geq m-g}p_{\xi|G}(x) dx dg \underset{m\rightarrow\infty}{=} o(e^{-cm}).
\]

Next we need a lower bound for $I_2$ : we denote $t_m = \frac{2log(m)}{\gamma c} - \frac{log(l(m))}{c}$. Then : 

\begin{align*}
    I_2 =& \int_0^{m+t_m} g p_G(g) \int_{x \geq m-g}p_{\xi|G}(x) dx dg + \int_{m+t_m}^\infty g p_G(g) \int_{x \geq m-g}p_{\xi|G}(x) dx dg\\
    \intertext{using this time the fact that $g$ is positive :}
    \geq& \underbrace{\int_0^{m+t_m} g p_G(g) \int_{x \geq m}p_{\xi|G}(x) dx dg}_{E_1} + \underbrace{\int_{m+t_m}^\infty g p_G(g) \int_{x \geq m-g}p_{\xi|G}(x) dx dg}_{E_2}\\
\end{align*}

We can handle $E_1$ with lemma \ref{lma:HTLT:conditionnal_converge_as} and a call to slutsky's theorem as $G\1[0,m+t_m]e^{cm} \underset{m\rightarrow\infty}{\overset{\mathcal{L}}{\rightarrow}}G\1[G \in [0,\infty) ]e^{cm}$:
\begin{align*}
    e^{cm}E_1 \underset{\phantom{m\rightarrow\infty}}{=}& \esp{G\1[G \in [0,m+t_m)]e^{cm}\proba{\xi > m}[G]} \\
    \underset{t\rightarrow\infty}{\rightarrow}& 0
\end{align*}

Then :
$E_1 \underset{m\rightarrow \infty}{=} o(e^{cm}) $

For $E_2$ : 
\begin{align*}
    E_2 \geq &\int_{m+t_m}^\infty gp_G(g)\int_{x\geq -t_m}p_{\xi|G}(x)dx dg \\
    \geq& (m+t_m)\int_{m+t_m}^\infty p_G(g)\int_{x\geq -t_m}p_{\xi|G}(x)dx dg \\
    =& (m+t_m)\int_{m+t_m}^\infty\int_{x\geq t_m}p_{G,\xi}(g,x)dx dg \\
    =& (m+t_m)\proba{G \geq m+t_m, \xi \geq -t_m}\\
    \intertext{here we use the fact that this joint probability is bounded below by Frechet's minimal copula :}
    \geq&(m+t_m)max(\proba{G \geq m+t_m} + \proba{\xi\geq -t_m}-1,0)\\
    \intertext{here we use the fact that for m big enough $\proba{\xi\leq -t_m} \leq \exp(-ct_m) = exp(-\frac{log(m)}{\gamma } + log(l(m))) = \frac{l(m)}{m^{2/\gamma}}$}
    \geq&(m+t_m)max(\proba{G \geq m+t_m} - \frac{l(m)}{m^{2/\gamma}},0)\\
    \geq& (m+t_m)max(\frac{l(m+t_m)}{(m+t_m)^{1/\gamma}} - \frac{l(m)}{m^{2/\gamma}},0)
\end{align*}

\begin{align*}
    \frac{l(m+t_m)}{(m+t_m)^{1/\gamma}} - \frac{l(m)}{m^{2/\gamma}} =& \frac{m^{2/\gamma}l(m+t_m) - (m+t_m)^{1/\gamma}l(m)}{(m+t_m)^{1/\gamma}m^{2/\gamma}}\\
    =& \frac{m^{2/\gamma}(l(m+t_m) - l(m)(\frac{1}{m} + t_m/m^{2})^{1/\gamma})}{(m+t_m)^{1/\gamma}m^{2/\gamma}}\\
    =& \frac{(l(m+t_m) - l(m)(\frac{1}{m} + t_m/m^{2})^{1/\gamma})}{m^{1/\gamma}(1+t_m/m)^{1/\gamma}}
\end{align*}

Last quantity is superior to 0 for $m$ sufficiently big (such that $\frac{l(m)}{m^{1/\gamma}} < l(m+t_m)$). In the precedent bound : 
\[
    (m+t_m)max(\frac{l(m+t_m)}{(m+t_m)^{1/\gamma}} - \frac{l(m)}{m^{2/\gamma}},0) \underset{m \rightarrow \infty}{=} (m+t_m)\frac{(l(m+t_m) - l(m)(\frac{1}{m} + t_m/m^{2})^{1/\gamma})}{m^{1/\gamma}(1+t_m/m)^{1/\gamma}}.
\]

We need now to consider the normalisation $\proba{M>m}$ :
\begin{align*}
\proba{M>m} \underset{\phantom{m \rightarrow \infty}}{=}& \int_\R p_G(g) \int_{x \geq m-g} p_{\xi|G}(x) dx dg\\
\underset{\phantom{m \rightarrow \infty}}{=}& \underbrace{\int_{\R^-} p_G(g) \int_{x \geq m-g}p_{\xi|G}(x) dx dg}_{I'_1} + \underbrace{\int_{0}^\infty p_G(g) \int_{x \geq m-g}p_{\xi|G}(x) dx dg}_{I'_2}.\\
\end{align*}
By the same argument as for $I_1$ we know that $I'_1 = o(e^{-sm})$. For $I'_2$, let $\lambda \in [0,1)$: 

\begin{align*}
I'_2 =& \underbrace{\int_0^{\lambda m} p_G(g) \int_{x \geq m-g}p_{\xi|G}(x) dx dg}_{E'_1} + \underbrace{\int_{\lambda m}^\infty p_G(g) \int_{x \geq m-g}p_{\xi|G}(x) dx dg}_{E'_2}.
\end{align*}

For $E'_1$ :
\begin{align*}
    \int_0^{\lambda m} p_G(g) \int_{x \geq m-g}p_{\xi|G}(x) dx dg \leq& \int_0^{\lambda m} p_G(g) \int_{x \geq (1-\lambda)m}p_{\xi|G}(x) dx dg \\
    =& \int_0^{\lambda m} p_G(g) \int_{x \geq (1-\lambda)m}p_{\xi|G}(x) dx dg \\
    =& \esp{\1[G \in [0,tm]]\proba{\xi >(1-t)m}[G]}\\
    =& o(e^{-c(1-\lambda)m})
\end{align*}

For $E'_2$ : 
\begin{align*}
    \int_{\lambda m}^\infty p_G(g) \int_{x \geq m-g}p_{\xi|G}(x) dx dg \leq& \int_{\lambda m}^\infty p_G(g) dg \\
    =& \frac{l(\lambda m)}{(\lambda m)^{1/\gamma}} \\
\end{align*}

We can conclude that $\forall \lambda \in [0,1)$ : 
\begin{align*}
\esp{G}[M>m] \underset{m \rightarrow \infty}{\geq}& (m+t_m)\frac{(l(m+t_m) - l(m)(\frac{1}{m} + t_m/m^{2})^{1/\gamma})}{m^{1/\gamma}(1+t_m/m)^{1/\gamma}} \times \frac{(\lambda m)^{1/\gamma}}{l(\lambda m) + o(exp(-c(1-\lambda)m))} + o(m)\\
\underset{m \rightarrow \infty}{\sim}& (m+t_m)\frac{\lambda^{1/\gamma}(l(m+t_m) - l(m)(\frac{1}{m} + t_m/m^{2})^{1/\gamma})}{l(m+t_m)} \\
\underset{ m \rightarrow \infty}{\sim}& \lambda^{1/\gamma}m 
\end{align*}
This is true for any $\lambda \in [0,1)$, hence : 
\[\esp{G}[M>m] \underset{m \rightarrow \infty}{\geq} m + o(m)\]
\end{proof}

%% file: Preprint_arxiv.bbl
\newcommand{\etalchar}[1]{$^{#1}$}
\begin{thebibliography}{BMTM23}

\bibitem[AOS{\etalchar{+}}16]{amodeiConcreteProblemsAI2016}
Dario Amodei, Chris Olah, Jacob Steinhardt, Paul Christiano, John Schulman, and Dan Man{\'e}.
\newblock Concrete {{Problems}} in {{AI Safety}}, 2016.

\bibitem[BMP{\etalchar{+}}25]{bengioInternationalAISafety2025}
Yoshua Bengio, S{\"o}ren Mindermann, Daniel Privitera, Tamay Besiroglu, Rishi Bommasani, Stephen Casper, Yejin Choi, Philip Fox, Ben Garfinkel, Danielle Goldfarb, Hoda Heidari, Anson Ho, Sayash Kapoor, Leila Khalatbari, Shayne Longpre, Sam Manning, Vasilios Mavroudis, Mantas Mazeika, Julian Michael, Jessica Newman, Kwan~Yee Ng, Chinasa~T. Okolo, Deborah Raji, Girish Sastry, Elizabeth Seger, Theodora Skeadas, Tobin South, Emma Strubell, Florian Tram{\`e}r, Lucia Velasco, Nicole Wheeler, Daron Acemoglu, Olubayo Adekanmbi, David Dalrymple, Thomas~G. Dietterich, Edward~W. Felten, Pascale Fung, Pierre-Olivier Gourinchas, Fredrik Heintz, Geoffrey Hinton, Nick Jennings, Andreas Krause, Susan Leavy, Percy Liang, Teresa Ludermir, Vidushi Marda, Helen Margetts, John McDermid, Jane Munga, Arvind Narayanan, Alondra Nelson, Clara Neppel, Alice Oh, Gopal Ramchurn, Stuart Russell, Marietje Schaake, Bernhard Sch{\"o}lkopf, Dawn Song, Alvaro Soto, Lee Tiedrich, Ga{\"e}l Varoquaux, Andrew Yao, Ya-Qin Zhang, Fahad Albalawi, Marwan Alserkal, Olubunmi Ajala, Guillaume Avrin, Christian Busch, Andr{\'e} Carlos Ponce de Leon~Ferreira {de Carvalho}, Bronwyn Fox, Amandeep~Singh Gill, Ahmet~Halit Hatip, Juha Heikkil{\"a}, Gill Jolly, Ziv Katzir, Hiroaki Kitano, Antonio Kr{\"u}ger, Chris Johnson, Saif~M. Khan, Kyoung~Mu Lee, Dominic~Vincent Ligot, Oleksii Molchanovskyi, Andrea Monti, Nusu Mwamanzi, Mona Nemer, Nuria Oliver, Jos{\'e} Ram{\'o}n~L{\'o}pez Portillo, Balaraman Ravindran, Raquel~Pezoa Rivera, Hammam Riza, Crystal Rugege, Ciar{\'a}n Seoighe, Jerry Sheehan, Haroon Sheikh, Denise Wong, and Yi~Zeng.
\newblock International {{AI Safety Report}}, 2025.

\bibitem[BMTM23]{birkstedtAIGovernanceThemes2023}
Teemu Birkstedt, Matti Minkkinen, Anushree Tandon, and Matti M{\"a}ntym{\"a}ki.
\newblock {{AI}} governance: Themes, knowledge gaps and future agendas.
\newblock {\em Internet Research}, 33(7):133--167, December 2023.

\bibitem[CA16]{clarkFaultyRewardFunctions2016}
Jack Clark and Dario Amodei.
\newblock Faulty reward functions in the wild, December 2016.

\bibitem[CG05]{caillaultEmpiricalEstimationTail2005}
Cyril Caillault and Dominique Gu{\'e}gan.
\newblock Empirical estimation of tail dependence using copulas: Application to {{Asian}} markets.
\newblock {\em Quantitative Finance}, 5(5):489--501, October 2005.

\bibitem[CIP10]{chorosCopulaEstimation2010}
Barbara Choro{\'s}, Rustam Ibragimov, and Elena Permiakova.
\newblock Copula {{Estimation}}.
\newblock In Piotr Jaworski, Fabrizio Durante, Wolfgang~Karl H{\"a}rdle, and Tomasz Rychlik, editors, {\em Copula {{Theory}} and {{Its Applications}}}, volume 198, pages 77--91. Springer Berlin Heidelberg, Berlin, Heidelberg, 2010.

\bibitem[DBM{\etalchar{+}}24]{dotanEvolvingAIRisk2024}
Ravit Dotan, Borhane {Blili-Hamelin}, Ravi Madhavan, Jeanna Matthews, and Joshua Scarpino.
\newblock Evolving {{AI Risk Management}}: {{A Maturity Model}} based on the {{NIST AI Risk Management Framework}}, 2024.

\bibitem[DF20]{derumignyConditionalEmpiricalCopula2020}
Alexis Derumigny and Jean-David Fermanian.
\newblock Conditional empirical copula processes and generalized dependence measures, 2020.

\bibitem[EH24]{el-mhamdiGoodhartsLawApplication2024}
El-Mahdi {El-Mhamdi} and L{\^e}-Nguy{\^e}n Hoang.
\newblock On {{Goodhart}}'s law, with an application to value alignment, 2024.

\bibitem[GBB{\etalchar{+}}23]{grosseMachineLearningSecurity2023}
Kathrin Grosse, Lukas Bieringer, Tarek~R. Besold, Battista Biggio, and Katharina Krombholz.
\newblock Machine {{Learning Security}} in {{Industry}}: {{A Quantitative Survey}}.
\newblock {\em IEEE Transactions on Information Forensics and Security}, 18:1749--1762, 2023.

\bibitem[GMT{\etalchar{+}}24]{godinotManipulationsAreAI2024}
Augustin Godinot, Erwan~Le Merrer, Gilles Tr{\'e}dan, Camilla Penzo, and Franois Ta{\"i}ani.
\newblock Under manipulations, are some {{AI}} models harder to audit?
\newblock In {\em 2024 {{IEEE Conference}} on {{Secure}} and {{Trustworthy Machine Learning}} ({{SaTML}})}, pages 644--664, Toronto, ON, Canada, April 2024. IEEE.

\bibitem[Goo75]{goodhartMonetaryRelationshipsView1975}
Charles Goodhart.
\newblock Monetary relationships : A view from threadneedle street.
\newblock {\em Papers in monetary economics 1975}, 1975.

\bibitem[GSH22]{gaoScalingLawsReward2022}
Leo Gao, John Schulman, and Jacob Hilton.
\newblock Scaling {{Laws}} for {{Reward Model Overoptimization}}, 2022.

\bibitem[HG20]{hennessyGoodhartsLawMachine2020}
Christopher Hennessy and Charles~A.E. Goodhart.
\newblock Goodhart's {{Law}} and {{Machine Learning}}.
\newblock {\em SSRN Electronic Journal}, 2020.

\bibitem[Hos96]{hoskinAwfulIdeaAccountability1996}
Keith Hoskin.
\newblock The ``{{Awful Idea}} of {{Accountability}}'': {{Inscribing People}} into the {{Measurement}} of {{Objects}}.
\newblock In {\em Accountability: {{Power}}, {{Ethos}} and the {{Technologies}} of {{Managing}}}. International Thomson Business Press, London, r. munro and j. mouritsen edition, 1996.

\bibitem[HPSL23]{hsiaGoodhartsLawApplies2023}
Jennifer Hsia, Danish Pruthi, Aarti Singh, and Zachary~C. Lipton.
\newblock Goodhart's {{Law Applies}} to {{NLP}}'s {{Explanation Benchmarks}}, 2023.

\bibitem[KHB{\etalchar{+}}23]{karwowskiGoodhartsLawReinforcement2023}
Jacek Karwowski, Oliver Hayman, Xingjian Bai, Klaus Kiendlhofer, Charlie Griffin, and Joar Skalse.
\newblock Goodhart's {{Law}} in {{Reinforcement Learning}}, 2023.

\bibitem[KTG24]{kwaCatastrophicGoodhartRegularizing2024}
Thomas Kwa, Drake Thomas, and Adri{\`a} {Garriga-Alonso}.
\newblock Catastrophic {{Goodhart}}: Regularizing {{RLHF}} with {{KL}} divergence does not mitigate heavy-tailed reward misspecification, 2024.

\bibitem[LKTF20]{levineOfflineReinforcementLearning2020}
Sergey Levine, Aviral Kumar, George Tucker, and Justin Fu.
\newblock Offline {{Reinforcement Learning}}: {{Tutorial}}, {{Review}}, and {{Perspectives}} on {{Open Problems}}, 2020.

\bibitem[Luc76]{lucasEconometricPolicyEvaluation1976}
Robert~E. Lucas.
\newblock Econometric policy evaluation: {{A}} critique.
\newblock {\em Carnegie-Rochester Conference Series on Public Policy}, 1:19--46, January 1976.

\bibitem[LZD21]{linUsingAdversarialAttacks2021}
Jieyu Lin, Jiajie Zou, and Nai Ding.
\newblock Using {{Adversarial Attacks}} to {{Reveal}} the {{Statistical Bias}} in {{Machine Reading Comprehension Models}}.
\newblock In {\em Proceedings of the 59th {{Annual Meeting}} of the {{Association}} for {{Computational Linguistics}} and the 11th {{International Joint Conference}} on {{Natural Language Processing}} ({{Volume}} 2: {{Short Papers}})}, pages 333--342, Online, 2021. Association for Computational Linguistics.

\bibitem[Man23]{manheimBuildingLessflawedMetrics2023}
David Manheim.
\newblock Building less-flawed metrics: {{Understanding}} and creating better measurement and incentive systems.
\newblock {\em Patterns}, 4(10):100842, October 2023.

\bibitem[MG18]{manheimCategorizingVariantsGoodharts2018}
David Manheim and Scott Garrabrant.
\newblock Categorizing {{Variants}} of {{Goodhart}}'s {{Law}}, 2018.

\bibitem[MSP{\etalchar{+}}17]{meurerSymPySymbolicComputing2017}
Aaron Meurer, Christopher~P. Smith, Mateusz Paprocki, Ond{\v r}ej {\v C}ert{\'i}k, Sergey~B. Kirpichev, Matthew Rocklin, {\relax Am}iT Kumar, Sergiu Ivanov, Jason~K. Moore, Sartaj Singh, Thilina Rathnayake, Sean Vig, Brian~E. Granger, Richard~P. Muller, Francesco Bonazzi, Harsh Gupta, Shivam Vats, Fredrik Johansson, Fabian Pedregosa, Matthew~J. Curry, Andy~R. Terrel, {\v S}t{\v e}p{\'a}n Rou{\v c}ka, Ashutosh Saboo, Isuru Fernando, Sumith Kulal, Robert Cimrman, and Anthony Scopatz.
\newblock {{SymPy}}: Symbolic computing in {{Python}}.
\newblock {\em PeerJ Computer Science}, 3:e103, January 2017.

\bibitem[SHKK22]{skalseDefiningCharacterizingReward2022}
Joar Skalse, Nikolaus H.~R. Howe, Dmitrii Krasheninnikov, and David Krueger.
\newblock Defining and {{Characterizing Reward Hacking}}, 2022.

\bibitem[Str97]{strathernImprovingRatingsAudit1997}
Marilyn Strathern.
\newblock `{{Improving}} ratings': Audit in the {{British University}} system.
\newblock {\em European Review}, 5(3):305--321, July 1997.

\bibitem[TH22]{taoriDataFeedbackLoops2022}
Rohan Taori and Tatsunori~B. Hashimoto.
\newblock Data {{Feedback Loops}}: {{Model-driven Amplification}} of {{Dataset Biases}}, September 2022.

\bibitem[TU22]{thomasRelianceMetricsFundamental2022}
Rachel~L. Thomas and David Uminsky.
\newblock Reliance on metrics is a fundamental challenge for {{AI}}.
\newblock {\em Patterns}, 3(5):100476, May 2022.

\bibitem[VOFA24]{vassilevAdversarialMachineLearning2024}
Apostol Vassilev, Alina Oprea, Alie Fordyce, and Hyrum Anderson.
\newblock Adversarial machine learning : A taxonomy and terminology of attacks and mitigations.
\newblock Technical Report NIST 100-2e2023, {National Institute of Standards and Technology (U.S.)}, Gaithersburg, MD, January 2024.

\bibitem[WRP22]{weertsAreThereExceptions2022}
Hilde Weerts, Lamb{\`e}r Royakkers, and Mykola Pechenizkiy.
\newblock Are {{There Exceptions}} to {{Goodhart}}'s {{Law}}? {{On}} the {{Moral Justification}} of {{Fairness-Aware Machine Learning}}, 2022.

\bibitem[ZH20]{NEURIPS2020_b607ba54}
Simon Zhuang and Dylan {Hadfield-Menell}.
\newblock Consequences of misaligned {{AI}}.
\newblock In H.~Larochelle, M.~Ranzato, R.~Hadsell, M.F. Balcan, and H.~Lin, editors, {\em Advances in Neural Information Processing Systems}, volume~33, pages 15763--15773. Curran Associates, Inc., 2020.

\end{thebibliography}
